\documentclass[letterpaper]{article} 
\usepackage{aaai24}  
\usepackage{times}  
\usepackage{helvet}  
\usepackage{courier}  
\usepackage[hyphens]{url}  
\usepackage{graphicx} 
\usepackage{natbib}  
\usepackage{caption} 
\usepackage{algorithm}
\usepackage{algorithmic}

\usepackage{newfloat}
\usepackage{listings}
\DeclareCaptionStyle{ruled}{labelfont=normalfont,labelsep=colon,strut=off} 
\floatstyle{ruled}
\newfloat{listing}{tb}{lst}{}
\floatname{listing}{Listing}
\usepackage{amsmath}
\usepackage{amssymb}
\usepackage{amsthm}
\usepackage{mathtools}
\theoremstyle{plain}
\newtheorem{lemma}{Lemma}
\newtheorem{theorem}{Theorem}
\newtheorem{remark}{Remark}
\newtheorem{corollary}{Corollary}[theorem]
\newtheorem{assumption}{Assumption}
\theoremstyle{definition}
\newtheorem{definition}{Definition}

\title{
New Classes of the Greedy-Applicable Arm Feature Distributions in the Sparse Linear Bandit Problem
}
\author{
    Koji Ichikawa\textsuperscript{\rm 1 2},
    Shinji Ito\textsuperscript{\rm 1 2 3},
    Daisuke Hatano\textsuperscript{\rm 3},
    Hanna Sumita\textsuperscript{\rm 4},\\
    Takuro Fukunaga\textsuperscript{\rm 5},
    Naonori Kakimura\textsuperscript{\rm 6},
    Ken-ichi Kawarabayashi\textsuperscript{\rm 7 8}
}
\affiliations{
    \textsuperscript{\rm 1}NEC Corporation\\
    \textsuperscript{\rm 2}National Institute of Advanced Industrial Science and Technology\\
    \textsuperscript{\rm 3}RIKEN AIP\\
    \textsuperscript{\rm 4}Tokyo Institute of Technology\\
    \textsuperscript{\rm 5}Chuo University\\
    \textsuperscript{\rm 6}Keio University\\
    \textsuperscript{\rm 7}National Institute of Informatics\\
    \textsuperscript{\rm 8}The University of Tokyo\\
    \{k\_ichikawa0, i-shinji\}@nec.com,
    daisuke.hatano@riken.jp, 
    sumita@c.titech.ac.jp,
    fukunaga.07s@chuo-u.ac.jp,\\
    kakimura@math.keio.ac.jp, 
    k\_keniti@nii.ac.jp
}

\begin{document}

\maketitle

\begin{abstract}
We consider the sparse contextual bandit problem where arm feature affects reward through the inner product of sparse parameters. Recent studies have developed sparsity-agnostic algorithms based on the greedy arm selection policy. However, the analysis of these algorithms requires strong assumptions on the arm feature distribution to ensure that the greedily selected samples are sufficiently diverse; One of the most common assumptions, relaxed symmetry, imposes approximate origin-symmetry on the distribution, which cannot allow distributions that has origin-asymmetric support. In this paper, we show that the greedy algorithm is applicable to a wider range of the arm feature distributions from two aspects. Firstly, we show that a mixture distribution that has a greedy-applicable component is also greedy-applicable. Second, we propose new distribution classes, related to Gaussian mixture, discrete, and radial distribution, for which the sample diversity is guaranteed. The proposed classes can describe distributions with origin-asymmetric support and, in conjunction with the first claim, provide theoretical guarantees of the greedy policy for a very wide range of the arm feature distributions.
\end{abstract}

\section{Introduction}
The contextual bandit problems are extensively investigated in various settings~\citep{lattimore2020bandit}, with practical applications to recommendations~\citep{DBLP:conf/www/LiCLS10}, clinical trials~\citep{DBLP:conf/mlhc/DurandAISMP18,DBLP:journals/ior/BastaniB20}, and many others~\citep{DBLP:journals/corr/abs-1904-10040}. 
The problems are sequential decision making problems where, in each round, a learner observes a set of arms with context, chooses one of them, and receives a corresponding reward. 
In this paper, we assume that arm features~(context) for each arm are stochastically generated and that the reward is affected by the inner product of the selected arm features and unknown parameters.
The problem is known to admit sublinear-regret algorithms that utilize the upper confidence bound for the arm selection criteria~\citep{DBLP:journals/jmlr/Auer02,DBLP:conf/colt/DaniHK08,DBLP:conf/www/LiCLS10,DBLP:journals/mor/RusmevichientongT10,DBLP:journals/jmlr/ChuLRS11,DBLP:conf/nips/Abbasi-YadkoriPS11}, 
but they usually allow the unknown parameters to be dense.

In the sparse linear bandits,
introduced by \citet{DBLP:journals/jmlr/Abbasi-YadkoriPS12} and \citet{DBLP:journals/jmlr/CarpentierM12},
we consider the situations where the arm features are high-dimensional, but the unknown parameters are sparse.
If we are given the sparsity of the unknown parameters in advance,
the sparse linear bandits admit sublinear-regret algorithms~\citep{DBLP:journals/jmlr/Abbasi-YadkoriPS12,DBLP:journals/jmlr/CarpentierM12,DBLP:journals/ior/BastaniB20,DBLP:conf/icml/WangWY18,DBLP:conf/nips/KimP19},
which outperform the linear bandit algorithms for the dense setting by exploiting the sparsity.
On the other hand, for the estimation of the unknown parameters, the arm features chosen by the algorithms must be sufficiently diverse.
In particular, recent algorithms~\citep{DBLP:journals/ior/BastaniB20,DBLP:conf/icml/WangWY18,DBLP:conf/nips/KimP19} adopt the forced-sampling or uniform sampling step to guarantee the diversity of the chosen arm features.
We need to know the sparsity in advance to ensure an optimal ratio between the forced or uniform sampling step and the other strategic arm selection step.

\citet{DBLP:conf/icml/OhIZ21} and \citet{DBLP:conf/icml/AriuAP22} recently proposed algorithms for the sparse linear bandits that work without knowing the sparsity in advance.
Such sparsity-agnostic algorithms are based on the greedy arm selection policy.
They showed that the greedy arm selection automatically guarantees the sample diversity under certain assumptions on the arm feature distribution.
However, the existing analysis of sparsity-agnostic algorithms is of limited applicability in practice due to their strong  assumption on the arm feature distribution.
One of their typical assumptions is relaxed symmetry,
which requires the arm feature distribution to be almost symmetric around the origin.
Thus the current analysis is not applicable to the problem where arm features take only positive values, arising in practical applications such as a recommender system.

In this paper, we show that the greedy algorithm is in fact applicable to the problem with a wider range of the arm feature distributions from two aspects. 
Firstly, we show that, if an arm feature distribution is greedy-applicable, then so does their mixture.
Here we call arm feature distribution greedy-applicable if the sample diversity is guaranteed under the greedy arm selection policy.
It is also shown in the proof that a larger proportion of the greedy-applicable component in the mixture distribution yields a tighter regret upper bound, motivating the presentation of a variety of greedy-applicable distribution classes.
Secondly, we propose new representational classes of the greedy-applicable distributions, related to Gaussian mixture, discrete, and radial distribution. The proposed classes can describe distributions with origin-asymmetric support.
These two generalizations provide theoretical guarantees of the greedy policy for a very wide range of the arm feature distributions.
Moreover, we demonstrate the usefulness of our analysis by applying it to the other cases: thresholded lasso bandit~\citep{DBLP:conf/icml/AriuAP22}, combinatorial setting, and non-sparse setting~\citep{DBLP:journals/ior/BastaniB20}.

The organization of this paper is as follows: in Section \ref{section:related_works}, we describe related work; in Section \ref{section:preliminary}, we introduce the regret analysis for the greedy policy on sparse linear bandits and the assumptions for the arm feature distribution; in Section \ref{section:extension}, we present our theorems on the greedy-applicable distributions. Section \ref{section:application} is devoted to applications of our analysis and, finally, we give a discussion and conclusion in Section \ref{section:conclusion}.

\section{Related Works}
\label{section:related_works}
In the sparse linear bandit problem, \citet{DBLP:journals/jmlr/Abbasi-YadkoriPS12} proposed an algorithm with the online-to-confidence-set conversion technique, giving a regret upper bound of $O(\sqrt{s_0dT})$. 
Here, $T$ is the horizon and $d$ is the dimension of the arm features, and $s_0$ is the sparsity of the unknown parameter. 
\citet{DBLP:journals/ior/BastaniB20} 
and \citet{DBLP:conf/icml/WangWY18}
dealt with a multi-parameter setting, i.e., a setting where each arm has an unknown sparse parameter and arm selection is performed for a  vector of context, giving
regret upper bounds of $O(s_0^2 (\log dT)^2)$ and 
$O(s_0^2 \log d \log T)$, respectively.
In the single-parameter setting, \citet{DBLP:conf/nips/KimP19} showed an $O(s_0 \log (dT) \sqrt{T})$ upper bound using the doubly-robust lasso bandit approach.

In the above series of studies, their algorithms 
used prior knowledge of $s_0$ as input to the algorithm.
\citet{DBLP:conf/icml/OhIZ21} proposed a sparsity-agnostic algorithm and, under the assumption for the arm feature distribution, called relaxed symmetry, gave a regret upper bound of $O(s_0 \sqrt{\log(dT) T})$.
The thresholded lasso bandit algorithm proposed by \citet{DBLP:conf/icml/AriuAP22} is another sparsity-agnostic algorithm. They showed an $O(\sqrt{s_0 T})$ upper bound under an additional assumption about the sparse positive definiteness of the arm feature distribution.
However, the relaxed symmetry imposed for their regret analysis was severe.
Our study shows that the greedy-based algorithm in fact is applicable to a wider class of the arm feature distributions.

Another line of research is to consider the greedy algorithm in the dense parameter settings.
In \citet{bastani2021mostly}, the greedy algorithm was studied in the multi-parameter setting. They showed that by introducing the covariate diversity assumption, the greedy algorithm for two arms achieves the rate optimal.
\citet{DBLP:conf/nips/KannanMRWW18} investigated
the greedy algorithm in the perturbed adversarial setting. The authors showed that even in adversarial samples, the addition of stochastic isotropic Gaussian perturbations makes the regret upper bound of $O(d \sqrt{d T})$.
For sparse linear bandit, \citet{sivakumar2020structured} 
gave an $O(\sqrt{s_0 d T})$ upper bound in the perturbed adversarial setting.

\section{Preliminary}
\label{section:preliminary}
\subsection{Problem Setup}
\label{subsec:problem}
In this article, we consider the following linear bandit problem: We are given a horizon $T$. For each round $t \in [T]$, we observe a set of $K$ arm features of $d$ dimensional vector $\mathcal{X}^t := \{X_{1}^{t}, \dots, X_{K}^{t}\}$, where $X_{i}^{t} \in \mathbb{R}^{d}$ for $ i \in [K]$, generated by an arm feature distribution $P(\mathcal{X}^t)$.
For each round $t$, we select one of these arms with index $a_t \in [K]$ and receive a corresponding reward $r_t \in \mathbb{R}$.

In the sparse linear contextual bandit, the observed rewards are modeled as the inner product of the selected arm feature $X_{a_t}^{t}$ and an unknown sparse parameter $\beta^* \in \mathbb{R}^{d}$ with sparsity $ \|\beta^* \|_0 = s_0$:
\begin{align}
    \label{eq:reward}
    r_{t} := X_{a_t}^{t\top} \beta^{*} + \epsilon_t
    \,,
\end{align}
where $\epsilon_t$ is a $\sigma$-sub-Gaussian noise satisfying 
$\mathbb{E}[e^{\lambda \epsilon_t} \mid \mathcal{F}_{t}] \leq e^{\lambda^2 \sigma^2/2}$ for any $\lambda \in \mathbb{R}$,
with $\mathcal{F}_{t}$ being the $\sigma$-algebra $\sigma(\mathcal{X}^1, a_1, r_1, \dots, \mathcal{X}^{t-1}, a_{t-1}, r_{t-1}, \mathcal{X}^t, a_t)$.
For later use, we here define another $\sigma$-algebra:
$\mathcal{F}'_{t}:=\sigma(\mathcal{X}^1, a_1, r_1, \dots, \mathcal{X}^t, a_t, r_t)$.
The problem is to minimize the following expected regret:
\begin{align}
    R(T) = \sum_{t=1}^{T} \mathbb{E} [r_{t}^{*} - r_{t}]
    \,,
\end{align}
under selection criteria $\{a_1, \dots, a_T\}$.
Here $r_{t}^{*}$ is the reward for the optimal arm choice in round $t$.

Below, we make the following standard assumptions about the bound for the arm feature $\mathcal{X} = \{X_{1}, \dots, X_{K}\} \sim P(\mathcal{X})$ and $\beta^*$.
\begin{assumption}
    \label{assumption:x_beta_bound}
    The arm feature $X_i$ for each $i \in [K]$ and $\beta^*$ are bounded as $\|X_i\|_{\infty} \leq x_{\mathrm{max}} < \infty$ and $\|\beta^*\|_{1} \leq b < \infty$, respectively.
\end{assumption}

\subsection{LASSO Estimator}
Here we introduce the theories of LASSO~\citep{tibshirani1996regression}, which is employed in most papers on the sparse linear bandits.
LASSO estimates the parameter $\beta^{*}$ under the observed samples $(X^{1}_{a_1}, r_{1}, \dots, X^{t}_{a_t}, r_{t})$ by the following $L^1$-regularized least square:
\begin{align}
    \label{eq:lasso_estimator}
    \hat{\beta}_{t} = 
    \mathop{\mathrm{argmin}}\limits_{\beta} 
    \frac{1}{t}
    \sum_{s=1}^{t} (r_s - X^{s\top}_{a_s} \beta)^2 + \lambda_t \|\beta\|_1
    \,, 
\end{align}
where $\lambda_{t} > 0$ is a hyperparameter that may depend on the round $t$ (i.e., sample size) and we also define $\hat{\beta}_0 = (0, \dots, 0)$ for $t=0$. 
We call $\hat{\beta}_t$ the LASSO estimator after the reward in round $t$. 

In the evaluation of the gap between the estimator $\hat{\beta}_t$ and true parameter $\beta^{*}$, the compatibility condition defined below plays an essential role, as seen in Lemma \ref{lemma:lasso_estimator}:
\begin{definition}[Compatibility Condition]
    We say that a positive semi-definite matrix $\Sigma \in \mathbb{R}^{d \times d}$ satisfies the compatibility condition if the following inequality holds for some compatibility constant $\phi > 0$ and some active set $\mathcal{S} \subset [d]$: 
    \begin{align}
        \frac{
            V^{\top} \Sigma V
        }{
            \| V \|_{\mathcal{S},1}^2
        }
            \geq  
                \frac{\phi^2}{|\mathcal{S}|}
        \,,
        \label{eq:def_cc}
    \end{align}
    for any
    $
        V \in \mathbb{R}^{d}
    $
    such that
    $
        \| V \|_{\mathcal{S}^c, 1} 
        \leq
        3 \| V \|_{\mathcal{S}, 1} 
    $
    .
    Here $\mathcal{S}^c := [d] \backslash \mathcal{S}$ and the norm $\| \cdot \|_{\mathcal{T}, 1}$ for a set $\mathcal{T} \subset [d]$ indicates the $L^1$-norm for the indices $i \in \mathcal{T}$, i.e., $\| V \|_{\mathcal{T}, 1} := \sum_{i \in \mathcal{T}} |V_i|$.
\end{definition}
In this paper, we fix the active set $\mathcal{S} = \{i \in [d] \mid \beta^{*}_i \neq 0\}$ and $|\mathcal{S}| = s_0$.
We also define a map 
$\phi_{\mathcal{S}}: \mathbb{R}^{d \times d} \rightarrow \mathbb{R}_{\geq 0}$
by
\begin{align}
    \phi_{\mathcal{S}}(\Sigma) :=
    \min_{V \in \mathcal{D}}
    \sqrt{
        |\mathcal{S}|
        \frac{
            V^{\top} \Sigma V
        }{
            \| V \|_{\mathcal{S},1}^2
        }
    }
    \,,
\end{align}
where 
$
    \mathcal{D} := 
        \{
            V \in \mathbb{R}^{d}
            \mid
            \| V \|_{\mathcal{S}^c, 1} 
            \leq
            3 \| V \|_{\mathcal{S}, 1} 
        \}
$.
The statement that $\Sigma$ satisfies the compatibility condition is equivalent to $\phi_{\mathcal{S}}(\Sigma) > 0$.
In the regret analysis of the sparse linear bandit, the empirical and expected Gram-matrix, given by the first and second definitions below respectively, are subject to the compatibility condition:
\begin{align}
    G_t := \frac{1}{t} \sum_{s=1}^{t} X_{a_{s}}^{s} X_{a_{s}}^{s\top}
    \,,
    \hspace{5pt}
    \bar{G}_t := \frac{1}{t} \sum_{s=1}^{t} \mathbb{E} [X_{a_{s}}^{s} X_{a_{s}}^{s\top} \mid \mathcal{F}'_{s-1}]
    \,.
\end{align}
Here we also define $\phi_t := \phi_{\mathcal{S}}(G_t)$ and $\bar{\phi} := \text{sup}\{\phi \geq 0 \mid \phi_{\mathcal{S}}(\bar{G}_t) \geq \phi, \forall t, a.s.\}$ for convenience.
We note that $\bar{\phi}$ depends on both the arm feature distribution and the arm selection policy. To clarify the dependence of arm feature distribution $P(\mathcal{X})$, we sometimes write the expected Gram-matrix as $\bar{G}_{t}(P)$.

Under Assumption \ref{assumption:x_beta_bound}, it is known that the following two inequalities hold with high probabilities.
\footnote{
    We give all the proofs in this paper in the appendix.
}
\begin{lemma}[Lemma 4 in \citet{DBLP:conf/icml/OhIZ21}]
    \label{lemma:epsX}
    Under Assumption \ref{assumption:x_beta_bound},
    for any $\delta > 0$, the following inequality holds:
    \begin{align}
        \label{eq:epsX}
        \frac{1}{t}
        \left\|
        \sum_{s=1}^{t} \epsilon_s X^s_{a_s}
        \right\|_{\infty}
        \leq
        x_{\mathrm{max}}
        \sigma \sqrt{
            \frac{
                \delta^2
                +
                2 \log d
            }{t}
        }
    \,,
    \end{align}
    with probability at least $1 - e^{- \delta^2 /2}$.
    Here $\epsilon_s$ is the $\sigma$-sub-Gaussian noise in Eq.~\eqref{eq:reward}.
\end{lemma}
\begin{lemma}[Corollary 2 in \citet{DBLP:conf/icml/OhIZ21}]
    \label{lemma:cc_btwn_empirical_expected}
    Under Assumption \ref{assumption:x_beta_bound},
    for $t \geq T_0 := \log(d (d-1))/\kappa(\bar{\phi})$
    where 
    $\kappa(\bar{\phi}) := \min (2 - \sqrt{2}, \bar{\phi}^2 / (256 x_{\mathrm{max}}^2 s_0))$,
    the following inequality holds:
    \begin{align}
        \phi_t^2 \geq \frac{\bar{\phi}^2}{2}
        \,.
        \label{eq:cc_btwn_empirical_expected}
    \end{align}
    with probability at least
    $1 - e^{-t \kappa(\bar{\phi})^2}$.
\end{lemma}
When the above two highly probable events hold, and if $\bar{\phi} > 0$, the gap between the true parameters $\beta^{*}$ and the estimator $\hat{\beta}_t$ can be bounded by the inverse of $\bar{\phi}$, as stated in the following lemma.
\begin{lemma}
    \label{lemma:lasso_estimator}
    If the inequalities \eqref{eq:epsX} and \eqref{eq:cc_btwn_empirical_expected} hold, and $\bar{\phi} > 0$, then by setting $\lambda_{t} \geq 4 x_{\mathrm{max}} \sigma \sqrt{ (\delta^2 + 2 \log d) / t} $, we have
    \begin{align}
        \| \beta^{*} - \hat{\beta}_t \|_{1} 
        \leq
        \frac{8 \lambda_t s_0}{\bar{\phi}^2}
        \,.
    \end{align}
\end{lemma}

\subsection{The Compatibility Constant in the Greedy Algorithm}
\label{subsec:regret}
We here present the contribution of the compatibility constant $\bar{\phi}$ in the greedy algorithm. The greedy algorithm chooses the arm $a_t$ that satisfies $a_t = \mathrm{argmax}_{k \in [K]} X_{k}^{t\top} \hat{\beta}_{t-1}$ for each round $t$, where $\hat{\beta}_{t-1}$ is the LASSO estimator defined in Eq.~\eqref{eq:lasso_estimator}.
Under the greedy policy, we obtain the following regret bound from Lemmas \ref{lemma:epsX}, \ref{lemma:cc_btwn_empirical_expected}, and \ref{lemma:lasso_estimator}:
\begin{lemma}
    \label{lemma:greedy_bound_empirical}
    Under Assumption \ref{assumption:x_beta_bound}, if $\bar{\phi} > 0$, the expected regret for the greedy algorithm is upper bounded by:
    \begin{align}
        R(T) 
        & \leq
            2 x_{\mathrm{max}} b
            \left(
            \frac{1 + \log (d (d-1))}{ \kappa(\bar{\phi})^2}
            +
            \frac{\pi^2}{3}
            \right)
        \nonumber \\
        &
            \hspace{20pt}
            +
            \frac{
                128 s_0 x_{\mathrm{max}}^2 \sigma
            }
            {
                \bar{\phi}^2
            }
            \sqrt{(4\log T + 2\log d) T}
        \,,
        \label{eq:regret_greedy}
    \end{align}
    where
    $\kappa(\bar{\phi}) := \min (2 - \sqrt{2}, \bar{\phi}^2 / (256 x_{\mathrm{max}}^2 s_0))$.
    \label{lemma:regret}
\end{lemma}
It follows from Eq.~\eqref{eq:regret_greedy} that $R(T) = O(\frac{1}{\bar{\phi}^2} \sqrt{(\log T + \log d) T})$ if $\bar{\phi}$ is non-zero. 
We note that in this regret analysis, no assumptions other than Assumption \ref{assumption:x_beta_bound} are made for the arm feature distribution.

In the following sections, we will discuss what arm feature distributions satisfy the condition $\bar{\phi} > 0$ in Lemma \ref{lemma:regret}.
Unfortunately, $\bar{\phi} > 0$ does not hold for all arm feature distributions. 
In the existing works, this has been shown in a restricted arm feature distribution where an approximate origin-symmetric condition is satisfied, as we will introduce in the next section. 
Our goal is to show that wider classes of the arm feature distribution satisfy $\bar{\phi} > 0$, and are {\it greedy-applicable} in the sense that they have a regret upper bound of Eq.~\eqref{eq:regret_greedy} for the greedy algorithm.

\subsection{Existing Assumptions for the Arm Feature Distribution}
\label{subsec:assumptions}
In this section, we present three assumptions employed in the existing studies for $\bar{\phi} > 0$, noting in particular that the relaxed symmetry (Assumption \ref{assumption:RS}) severely restricts the arm feature distribution.

The first assumption states that the arm feature distribution must have some diversity:
\begin{assumption}
    \label{assumption:CC}
    The expected Gram-matrix with random arm selection satisfies the compatibility condition for a compatibility constant $\phi_0 > 0$:
    \begin{align}
        \phi_{\mathcal{S}} \left( 
            \frac{1}{K}
            \sum_{k=1}^{K}
            \mathbb{E}
            \left[
                X_k X_k^{\top}
            \right]
        \right)
        >  \phi_0
    \end{align}
\end{assumption}
This assumption is commonly employed in sparse linear bandit algorithms, including the greedy algorithm, and does not strongly constrain the arm feature distribution. We note, however, that this assumption alone does not guarantee sample diversity (i.e., $\bar{\phi} > 0$) under the greedy arm selection policy.

Secondly, the following approximate origin-symmetric condition is imposed for the arm feature distribution $P(X_1, \dots, X_K)$ supported by $\mathop{\mathrm{supp}}(P) \subset (\mathbb{R}^d)^K$~\citep{DBLP:conf/icml/OhIZ21,DBLP:conf/icml/AriuAP22}:
\begin{assumption}[Relaxed Symmetry (RS)]
    \label{assumption:RS}
    If $\{X_1, \dots, X_K\} \in \mathop{\mathrm{supp}}(P)$, then $\{-X_1, \dots, -X_K\} \in \mathop{\mathrm{supp}}(P)$. Moreover, there exists $1 \leq \nu < \infty$ that satisfies $P(-X_1, \dots, -X_K) / P(X_1, \dots, X_K) \leq \nu$ for any $\{X_1, \dots, X_K\} \in \mathop{\mathrm{supp}}(P)$.
\end{assumption}
We stress that distributions with relaxed symmetry are limited. 
In particular, it cannot be applied to cases with origin-asymmetric supports. For example, this assumption cannot include arm features that only take positive values or are represented by origin-asymmetric discrete values. 
One of the main motivations for our research is the relaxation of this origin-symmetric assumption.
\begin{remark}
    \label{remark:RS}
    One might think that the arm feature distribution could be made to satisfy the origin-symmetric support by transforming the features, for example, by constant shift or scale transformation. However, when two or more axes of a feature distribution have asymmetric support that is correlated, feature engineering generally cannot guarantee the RS condition. 
\end{remark}
A simple case is when labels have a hierarchical structure. For example, if two binary labels show mammal and dog, respectively, (0, 1) meaning "not mammal but dog" is not possible. Transformations that mix axes cannot be performed because they break the sparse structure, and constant shifts and scale transformations for each axis cannot provide origin symmetric support for the discrete distribution that takes values of (0,0), (1,0), (1,1).
Such a label structure is common in the recommender system, where the category of items often has a hierarchical structure (e.g., books \textrightarrow academic books \textrightarrow computer science). 

While in the case of $K=2$, the above assumptions ensure $\bar{\phi} > 0$ for the greedy algorithm, the following third assumption is further required in the case of $K>2$:
\begin{assumption}[Balanced Covariance]
    \label{assumption:BC}
    For any $i \in \{2, \dots, K-1 \}$, any permutation $\pi: [K] \rightarrow [K]$,
    and any fixed $\beta \in \mathbb{R}^d$, there exists a constant $C_{\mathrm{BC}} < \infty$ that satisfies: 
    \begin{align}
        &
        \mathbb{E} \left[
            X_{\pi(i)} X_{\pi(i)}^{\top}
            I[
                X_{\pi(1)}^{\top} \beta
                \leq
                \dots
                \leq
                X_{\pi(K)}^{\top} \beta
            ]
        \right]
        \nonumber
        \\
        &
        \hspace{20pt}
        \preceq
        C_{\mathrm{BC}}
        \mathbb{E} \left[
            (
            X_{\pi(1)} X_{\pi(1)}^{\top}
            +
            X_{\pi(K)} X_{\pi(K)}^{\top}
            )
        \right.
        \nonumber 
        \\
        &
        \hspace{75pt}
        \left.
            I[
                X_{\pi(1)}^{\top} \beta
                \leq
                \dots
                \leq
                X_{\pi(K)}^{\top} \beta
            ]
        \right]
        \,.
    \end{align}
\end{assumption}
\citet{DBLP:conf/icml/OhIZ21} has shown that if the arm features are generated i.i.d., the coefficients $C_{\text{BC}}$ are of finite value, but of the exponential order of $K$ for general distributions. They conjecture that the coefficients would not be so large from the observation of the experimental results.
We give the proof of $\bar{\phi} > 0$ under Assumptions \ref{assumption:CC}, \ref{assumption:RS} and \ref{assumption:BC} in the appendix.

\section{Investigation of the Greedy-Applicable Distributions}
\label{section:extension}
In this section, we show that the applicability of the greedy algorithm can be extended to wider classes of arm feature distributions by proposing the following two aspects: 1) distributions having a mixture component of a greedy-applicable distribution are also greedy-applicable (Theorem \ref{theorem:cc_mixture_general}), and 2) several representational function classes are greedy-applicable distributions (Section 
\ref{subsec:basic_assumptions} and \ref{subsec:bases}).
\begin{remark}
    \label{remark:assumptions}
    As noted at the end of Section \ref{subsec:regret}, the regret analysis of Lemma \ref{lemma:regret} does not require assumptions about the arm feature distribution other than Assumption \ref{assumption:x_beta_bound}. 
    Assumptions \ref{assumption:CC}, \ref{assumption:RS} and \ref{assumption:BC} are used solely to show that $\bar{\phi} > 0$.
    Therefore, Assumptions \ref{assumption:CC}, \ref{assumption:RS} and \ref{assumption:BC} can be replaced by other assumptions that lead to $\bar{\phi} > 0$ without changing the regret upper bound in Lemma \ref{lemma:regret}.
\end{remark}
We note that while our analysis focuses on the minimum value of the compatibility constant, $\phi_{\mathcal{S}}$, it can easily be replaced by operators such as minimum eigenvalue, or restricted minimum eigenvalue for the matrix, which also measure the diversity of a matrix.

We conduct our analysis under the following assumption for the arm selection policy:
\footnote{
    We note that the assumption for the policy can be more general:
    $
        P(\text{Select $i$} \mid \mathcal{X}^t, \mathcal{B}^{t-1}) = f_i(\beta_{1, t-1}^{\top} X_1^{t}, \dots, \beta_{K, t-1}^{\top} X_K^{t})
    $.
    where $\mathcal{B} := \{\beta_{1, t-1}, \dots \beta_{K, t-1}\} \in \Theta \subset (\mathbb{R}^{d})^{K}$.
    The subsequent theorems are easy to extend, and the greedy policy is considered here for the sake of clarity.
}
\begin{assumption}
    \label{assumption:arm_selection}
    Under given arm features $\mathcal{X}^t = \{X_1^t, \dots, X_K^t \} \in (\mathbb{R}^{d})^{K}$ at round $t$, the arm selection probability for arm $i$ is described by the greedy policy:
    \begin{align}
    P(\text{Select $i$} \mid \mathcal{X}^t, \beta_{t-1}) 
    &
    \nonumber \\
    &
    \hspace{-50pt}
    :=
    \frac{
        \prod_{j \neq i} I\left[
            \beta_{t-1}^{\top} X_i^t \geq \beta_{t-1}^{\top} X_j^t
        \right]
    }{
        \sum_{i'=1}^{K} \prod_{j \neq i'} 
        I \left[
            \beta_{t-1}^{\top} X_{i'}^t \geq \beta_{t-1}^{\top} X_j^t
        \right]
    }
    \,, 
    \end{align}
    where the parameter $\beta_{t-1} \in \mathbb{R}^d$ is determined from information prior to $\mathcal{X}^t$ and $r^t$.
    The denominator is for random selection when tying occurs.
\end{assumption}

We define the class of arm feature distributions for which sample diversity is guaranteed under Assumption \ref{assumption:arm_selection} as follows:
\begin{definition}
    The arm feature distribution $P(\mathcal{X})$ is a \textit{$\phi_0$-greedy-applicable distribution} if, under Assumption \ref{assumption:arm_selection}, there exists $\phi_0 > 0$ such that the expected Gram-matrix $\bar{G}_t$ satisfies $\phi_{\mathcal{S}} (\bar{G}_t) > \phi_0$.
\end{definition}
Then, the following key property holds for the positivity of the compatibility constant:
\begin{theorem}
    \label{theorem:cc_mixture_general}
    If an arm feature distribution $P(\mathcal{X})$ is a mixture of a PDF $Q(\mathcal{X})$ and a $\phi_0$-greedy-applicable distribution $\tilde{P}(\mathcal{X})$, i.e., $P(\mathcal{X}) = c \tilde{P} (\mathcal{X}) + (1-c) Q(\mathcal{X})$ for a constant $0<c<1$, then $P(\mathcal{X})$ is a $c\phi_0$-greedy-applicable distribution.
\end{theorem}
The theorem indicates that the proof that a class of PDF is greedy-applicable gives a theoretical guarantee to a very wide range of distributions that have this class as their mixture component.\footnote{
The coefficient for specific $\tilde{P}$s and situations where $\tilde{P}$ approximates $P$ are presented in the appendix.
}
Currently, the only general greedy-applicable distributions are the ones introduced in the previous section (i.e., distributions that satisfy Assumption \ref{assumption:CC}, \ref{assumption:RS}, and \ref{assumption:BC}).
Below, we propose several new greedy-applicable classes and show the wide applicability of the greedy algorithm.

\subsection{Basic Assumptions for the Arm Feature Distribution}
\label{subsec:basic_assumptions}

Here, we introduce two key assumptions for our analysis: 1) at least one arm distribution is independent of the others, and 2) such an arm will be selected by the algorithm with a positive probability.
Intuitively, choosing this independent arm contributes to exploring the arm features if it generates diverse samples.
In other words, these assumptions simplify the analysis of the greedy algorithm's applicability by reducing it to a discussion on the feature distribution of the single independent arm, which will be discussed in Section 4.2.

Formally, these two assumptions are described as follows:
\begin{assumption}
    \label{assumption:indep_arm}
    There exists at least one arm $i \in [K]$ that is independent of the other arms:
    $
        P(\mathcal{X}) 
            =
            P(\mathcal{X} \backslash \{X_i\}) 
            P_i(X_i)
    $.
\end{assumption}
\begin{assumption}
    \label{assumption:selection_possibility}
    For $i$ defined in Assumption \ref{assumption:indep_arm}, the marginalized arm selection probability has a positive lower bound with respect to $\beta \in \mathbb{R}^d$ under the greedy policy given in Assumption \ref{assumption:arm_selection}:
    $
        \inf_{\beta \in \mathbb{R}^d} P(\text{Select $i$} \mid \beta)  > 0
    $.
\end{assumption}
The second assumption is made to avoid the possibility that arm $i$ is never selected under a certain $\beta$. If $\beta$ satisfying $ P(\text{Select $i$} \mid \beta)  = 0$ exists, then constraining $P_i$ alone cannot guarantee the sample diversity, as in the worst case, any $X \sim P_i(X)$ will not be sampled.

To clarify the requirements of the assumptions, we give two application examples below:

\paragraph{Example 1}
The simplest example satisfying Assumptions \ref{assumption:indep_arm} and \ref{assumption:selection_possibility} is when all arm features are generated independently from the same distribution. 
For instance, consider a case where a set of recommendation candidates is given each week, and the recommendation system selects an item from that set and measures the click-through rate. 
Suppose each candidate is selected independently and uniformly at random from all items.
If we assume that the expectation of the click-through rate is linear with respect to the item features, maximizing the cumulative click-through rate can be related to our bandit problem.
In this scenario, it is obvious that Assumption \ref{assumption:indep_arm} holds and, since all candidates are selected by the same probability, Assumption \ref{assumption:selection_possibility} also holds.

\paragraph{Example 2}
Another example is the case where all arm features are generated independently but from different distributions. 
Consider a book recommendation that has categorical tags of science and fiction.
We consider two arms (i.e., two recommendation candidates), where the first and second candidates are selected from independent uniform distributions of science books and fiction books, respectively.
Suppose the set of the fiction books includes sci-fi books with both science and fiction tags.
In this scenario, Assumption \ref{assumption:indep_arm} again obviously holds. Moreover, since the fiction books include the sci-fi books with the science tag, the candidate of the `fiction' arm can be selected even if the greedy algorithm heavily favors the science tag. 
Therefore, Assumption \ref{assumption:selection_possibility} also holds for the `fiction' arm.

\begin{remark}
\label{remark:relation}
We note the relation of our assumptions to the existing assumptions: Assumption \ref{assumption:RS} allows correlation between all arms and does not request Assumption \ref{assumption:selection_possibility}. 
However, as mentioned previously, Assumption \ref{assumption:RS} does not allow distributions with asymmetric support. 
Also, for non-i.i.d. arms, it is necessary to assume Assumption \ref{assumption:BC}.
On the other hand, Assumptions \ref{assumption:indep_arm} and \ref{assumption:selection_possibility} enable us to discuss the greedy-applicability of the arm feature distributions that cannot be handled by the conventional analysis.
We also note that it is sufficient for the arm feature distribution to satisfy either the conventional assumptions or
those we propose. This condition can be further relaxed by Theorem \ref{theorem:cc_mixture_general} to the claim that such a distribution is only required to be a component of a mixture.
\end{remark}

As we have mentioned at the beginning of this section, we will focus on the feature distribution of the single independent arm in the next section.
We define the distribution $P_i(X)$ that ensures the greedy-applicability of $P(\mathcal{X})$ in the following term:
\begin{definition}
    We call $P_i(X)$ the \textit{basis of the greedy-applicable distribution}, if there exists a positive constant $\phi_0 > 0$ and $P(\mathcal{X})$ is a $\phi_0$-greedy-applicable distribution under Assumptions \ref{assumption:arm_selection}, \ref{assumption:indep_arm}, and \ref{assumption:selection_possibility}.
\end{definition}

In the next section, we provide several bases, omitting index $i$ in $P_i (X_i)$ for brevity.

\subsection{Proposal for Several Bases}
\label{subsec:bases}
For the sake of our analysis, we first define the following time-independent expected Gram-matrix :
\begin{align}
    \tilde{G}_{\beta}
    :=
        \sum_{k} \int 
        X_{k} X_{k}^{\top} 
        P(\text{Select $k$} \mid \mathcal{X}, \beta) 
        P(\mathcal{X})
        \prod_{k'=1}^{K} 
        dX_{k'}
\,.
\end{align}
In this section, the analysis is performed for $\tilde{G}_{\beta}$ instead of $\bar{G}_t$, according to the following lemma:
\begin{lemma}
    \label{lemma:Gt}
    Under Assumption \ref{assumption:arm_selection}, if $\phi_{\mathcal{S}}(\tilde{G}_{\beta}) \geq \phi_0$ for any $\beta \in \mathbb{R}^d$, then $\phi_{\mathcal{S}}(\bar{G}_t) \geq \phi_0$.
\end{lemma}

\subsubsection{Gaussian Mixture Basis}
First, we propose a basis where $P(X)$ can be decomposed by a sum of finite Gaussian distributions.
\begin{definition}
    Gaussian mixture basis $P_{GM}(X)$ is a PDF that can be decomposed by a sum of finite Gaussian distribution: $ P_{GM}(X) = \sum_{n=1}^{N} w_n \mathcal{N}(X \mid \mu_n, \Sigma_n)$, where $ \mu_n \in \mathbb{R}^{d}$ is a mean vector and $\Sigma_n \in \mathbb{R}^{d \times d}$ is a positive definite covariance matrix for each Gaussian distribution. 
    The weight $w_n > 0$ satisfies $\sum_n w_n = 1$.
\end{definition}
Then, the lower bound for $\tilde{G}_{\beta}$ is given by the following theorem:
\begin{lemma}
    \label{lemma:gm_lower_bound}
    Under Assumption 
    \ref{assumption:arm_selection},
    \ref{assumption:indep_arm},
    and
    \ref{assumption:selection_possibility},
    if $P_i(X)$ is the Gaussian mixture basis $P_{GM} (X)$, 
    then the following lower bound holds:
    \begin{align}
        \tilde{G}_{\beta} \succeq 
        \sum_{n=1}^{N}
        w_n c_n(\beta)
        (\Sigma_n + \mu_n \mu_n^{\top})
        \,,
    \end{align}
    where $c_n(\beta) > 0$ is a $\beta$-dependent positive constant.
\end{lemma}
Since $\Sigma$ is a positive definite matrix, the following is obvious from Lemmas \ref{lemma:Gt} and \ref{lemma:gm_lower_bound}:
\begin{theorem}
    \label{theorem:gaussian_mixture}
    $P_{GM}(X)$ is a basis of the greedy-applicable distribution.
\end{theorem}

\subsubsection{Low-rank Gaussian Mixture and Discrete Basis}
Lemma \ref{lemma:gm_lower_bound} shows that $\tilde{G}_{\beta}$ can be bounded by a weighted sum of the second moments of each Gaussian component. Importantly, if we allow $c_n(\beta) = 0$, this lower bound also holds in the limit where the eigenvalues of $\Sigma_n$ are taken to be zero.
We provide a useful lemma for the coefficient $c_n (\beta)$:
\begin{lemma}
    \label{lemma:gmm_coef}
    $c_n(\beta) = 0$ if and only if
    the following two conditions hold:
    \begin{align}
        \beta^{\top} \Sigma_n \beta = 0
        \,,
        \text{and}
        \hspace{5pt}
        P(\text{Select $i$} \mid X_i = \mu_n, \beta) = 0
        \,,
    \end{align}
    where 
    $P(\text{Select $i$} \mid X_i, \beta)$ is the marginalized probability distribution for $\{X_j \mid j \in [K]\backslash \{i\}\}$.
\end{lemma}
\begin{corollary}
    \label{corollary:gmm_coef}
    $\lim_{\beta^{\top} \Sigma_n \beta \rightarrow 0} c_n(\beta) = P(\text{Select $i$} \mid X_i = \mu_n, \beta)$.
\end{corollary}

The limit operation allows us to include the Gaussian mixture distribution with a low-rank covariance matrix and discrete distribution in the theory.
For a positive semi-definite matrix $\Sigma^{d'} \in \mathbb{R}^{d \times d}$ with rank $d' \leq d$ and diagonalized 
by an orthogonal matrix $R \in \mathbb{R}^{d \times d}$
as $R^{\top} \Sigma^{d'} R = \mathrm{diag} (\lambda_1, \dots, \lambda_{d'}, 0, \dots, 0)$ where $\lambda_i > 0$ for $i \in [d']$,
we define the low-rank Gaussian distribution as follows:
\begin{align}
    \tilde{\mathcal{N}}(X \mid \mu, \Sigma^{d'})
    := 
    &
    \nonumber \\
    &
    \hspace{-50pt}
    \frac{
        \prod_{j=d'+1}^{d} 
        \delta((R X - R \mu)_j) 
    }{
        (2 \pi)^{d'/2} 
        (
        \prod_{i=1}^{d'} \lambda_i
        )^{1/2}
    }
    e^{
        -\frac{1}{2} 
        (X - \mu)^T \Lambda^{d_n} (X - \mu) 
    }
    \,,
\end{align}
where $\Lambda^{d'} := R\, \mathrm{diag} (\lambda_1^{-1}, \dots, \lambda_{d'}^{-1}, 0, \dots, 0) R^{\top}$ and $\delta(x)$ is the delta function.
We define the low-rank Gaussian mixture basis as follows:
\begin{definition}
    Low-rank Gaussian mixture basis $P_{LGM}(X)$ is a PDF that can be decomposed as a sum of finite low-rank Gaussian distribution: $P_{LGM}(X) = \sum_{n=1}^{N} w_n \tilde{\mathcal{N}}(X \mid \mu_n, \Sigma_{n}^{d_n})$.
    Moreover, $P_{LGM}$ satisfies the following condition for a positive constant $\phi_0 > 0$:
    \begin{align}
        \label{eq:LGM_cond}
        \inf\limits_{\beta \in \mathbb{R}^d}
        \phi_{\mathcal{S}}
        \left(
            \sum_{n} 
            c_n(\beta)
            w_n (\Sigma_{n}^{d_n} + \mu_{n} \mu_{n}^{\top}) 
        \right)
        \geq
        \phi_0
    \,.
    \end{align}
\end{definition}
Then, we obtain the following theorem:
\begin{theorem}
    \label{theorem:low_rank_gaussian_mixture}
    $P_{LGM}(X)$ is a basis of the greedy-applicable distribution.
\end{theorem}
We note that not all Gaussian mixtures with low-rank covariance matrix are included in the class, and that the condition indicated in Eq.~\eqref{eq:LGM_cond} is necessary for their second moment.

Trivially, the limit operation $\Sigma_n \rightarrow 0$ can represent a discrete distribution.
\begin{definition}
    Discrete basis $P_{D}(X)$ is a PDF that can be described by a discrete probability distribution: $P_{D}(X=\mu_n) = p_n$, where $p_n > 0$ satisfies $\sum_n p_n = 1$,
    and $\mu_n \in \mathbb{R}^{d}$ is an element of a set $\mathcal{M} := \{\mu_1, \dots, \mu_N \}$.  
    Moreover, $P_{D}$ satisfies the following condition for a positive constant $\phi_0 > 0$:
    \begin{align}
        \inf\limits_{\beta \in \mathbb{R}^d}
        \phi_{\mathcal{S}}
        \left(
            \sum_{n} 
            c_n(\beta)
            p_n \mu_{n} \mu_{n}^{\top}
        \right)
        \geq
        \phi_0
    \,.
    \end{align}
\end{definition}
\begin{corollary}
    \label{corollary:discrete}
    $P_{D}(X)$ is a basis of the greedy-applicable distribution.
\end{corollary}

\subsubsection{Radial Basis}
Due to the nature of the Gaussian distribution, 
$P_{GM}$ and $P_{LGM}$ cannot include distributions with truncation. As a basis for the truncated distribution, we consider PDF of radial function.
\begin{definition}
    Radial mixture basis $P_R(X)$ is a PDF that can be decomposed by a sum of finite radial distributions: $ P_{R}(X) = \sum_{n=1}^{N} w_n Q_n(X \mid \mu_n)$, where $ \mu_n \in \mathbb{R}^{d}$ is a vector and PDF $Q_n(X \mid \mu_n)$ is a radial function: $Q_n(X \mid \mu_n) = f_n(\|X - \mu_n\|_2)$ that satisfies $\int f_n(\|X\|_2) dX = 1$. 
\end{definition}
The radial function $f_n$ can include, for example, truncated uniform distribution and truncated standard normal distribution. 
Then, the following theorem holds:
\begin{theorem}
    \label{theorem:radial}
    $P_R(X \mid \mu)$ is a basis of the greedy-applicable distribution.
\end{theorem}
We summarize the formulae for the $\phi_{\mathcal{S}} (\bar{G}_t)$ lower bound of each basis in the appendix.

We here mention the relationship between our analysis and the smoothed analysis~\citep{DBLP:conf/nips/KannanMRWW18,sivakumar2020structured}.
In the smoothed analysis, it is assumed that the arm features are generated adversarially, and then observed with stochastic perturbations from a truncated isotropic normal distribution.
\citet{sivakumar2020structured} showed that the greedy algorithm in the sparse linear bandit can achieve an $O(\sqrt{sdT})$ regret upper bound in the perturbed adversarial setting.
The radial basis $P_R(X) = \sum_{n=1}^{N} w_n Q_n (X \mid \mu_n)$ is considered to correspond to a specific stochastic setting of the smoothed analysis.
That is, we can regard the radial basis as a distribution where the arm feature $\mu_n$ is chosen with probability $w_n$ and then perturbed by $Q_n (X \mid \mu_n)$.
In our analysis, we include general radial distributions other than the truncated isotropic normal distribution and also examine the applicability of the greedy algorithm for $P_{GM}$, $P_{LGM}$, and $P_{D}$.
While we expect these bases to retain their properties even in the perturbed adversarial setting, we leave a more detailed analysis for future work.

\subsection{Examples}
In this section, we demonstrate how the theorems of the previous section apply and show the greedy-applicability to specific examples that could not be dealt with in the analysis of previous studies.
In addition, we also performed a numerical experiment with artificial data to empirically validate our claim, which is given in the appendix.
Below, we use the following lemma.
\begin{lemma}
    \label{lemma:cc_sum}
    Let us define positive semi-definite matrices $\Lambda_n \in \mathbb{R}^{d \times d}$ and positive coefficients $w_n >0$, $w'_n > 0$ for $n \in [N]$. 
    If 
    $\phi_{\mathcal{S}}(\sum_{n=1}^{N} w_n \Lambda_n) > 0$, 
    then
    $\phi_{\mathcal{S}}(\sum_{n=1}^{N} w'_n \Lambda_n) > 0$. 
\end{lemma}

Consider that each arm has two-dimensional binary feature $(x_1, x_2)$ and that in arm $i$, each binary combination is realized with the following non-zero probability:
$p_1$ for $\mu_1 = (0, 0)$,
$p_2$ for $\mu_2 = (1, 0)$,
$p_3$ for $\mu_3 = (0, 1)$,
and
$p_4$ for $\mu_4 = (1, 1)$,
where 
$p_1 + p_2 + p_3 + p_4 = 1$.
We also assume that 
arm $i$ is independent of the other arms (for Assumption \ref{assumption:indep_arm})
and
$\inf_{\beta} P(\text{Select $i$} \mid X_i, \beta) > 0$ under the realization of $X_i = \mu_2$, $\mu_3$, or $\mu_4$ in the greedy algorithm (for Assumption \ref{assumption:selection_possibility}, Lemma \ref{lemma:gmm_coef}, and Corollary \ref{corollary:gmm_coef}).
Since $X_i$ does not take negative values, it is clearly a distribution for which Assumption \ref{assumption:RS} is not valid.
In our analysis, from Lemma \ref{lemma:gm_lower_bound},
we obtain
$
    \tilde{G}_{\beta}
    \succeq
    \sum_{n=1}^{4} 
    c_n(\beta)
    p_n \mu_{n} \mu_{n}^{\top}
    \succeq
    \sum_{n=2}^{4} 
    c_n(\beta)
    p_n \mu_{n} \mu_{n}^{\top}
$,
and from Corollary \ref{corollary:gmm_coef}, $\inf_{\beta} c_n(\beta) > 0$ for $n=2, 3$ and $4$.
Meanwhile,
\begin{align}
\sum_{n=2}^{4} 
p_n \mu_{n} \mu_{n}^{\top}
=
    \begin{pmatrix}
        p_2 + p_4 & p_4 \\
        p_4 & p_3 + p_4 \\
    \end{pmatrix}
\end{align}
is a positive definite matrix.
Then, Lemma \ref{lemma:cc_sum} derives that 
there exists $\phi_0 > 0$ such that
$\inf_{\beta} \phi_{\mathcal{S}} (\sum_{n=2}^{4} c_n(\beta) p_n \mu_n \mu_n^{\top}) > \phi_0$,
which indicates that the distribution of this example is $\phi_0$-greedy applicable.
We note that the same argument holds for $p_3=0$, which is the case shown in Remark \ref{remark:RS} that cannot satisfy Assumption \ref{assumption:RS} even with the constant shift.

As a next example, consider $d=2$, two independent arms, and both features uniformly distributed within the region $\{(x_1, x_2) \mid 0 \leq x_1 \leq 1, 0 \leq x_2 \leq 1, \sqrt{x_1^2 + x_2^2} \geq 0.1\}$. Again, this arm feature distribution does not satisfy Assumption \ref{assumption:RS} even with the constant shift.
In our analysis, we see that 
Assumptions \ref{assumption:indep_arm} and \ref{assumption:selection_possibility} are satisfied.
Moreover, the distribution of both arms is a mixture of the radial basis $Q (X) = f(\|X - (1/2, 1/2)\|_2)$ and the remaining, where $f(r):\mathbb{R}_{+} \rightarrow \mathbb{R}_{+}$ satisfies $f(r) = 4 / \pi$ for $0 \leq r \leq 1/2$ and $0$ otherwise.
Therefore, from Theorems \ref{theorem:cc_mixture_general} and \ref{theorem:radial}, this arm feature distribution is a (some) $\phi_0$-greedy applicable distribution.

\section{Application to Several Algorithms}
\label{section:application}
Our analysis can be applied to many existing algorithms. We illustrate some of them in this section.
While the application examples in this section focus on the sparse settings, an application to the dense parameter setting is also shown in the appendix.

\paragraph{Greedy algorithm~\citep{DBLP:conf/icml/OhIZ21}}
In the analysis for the greedy algorithm, Assumptions \ref{assumption:CC}, \ref{assumption:RS} and \ref{assumption:BC} are imposed on the arm feature distribution to guarantee $\bar{\phi} > 0$ (Lemma 10 of \citet{DBLP:conf/icml/OhIZ21}).
Our assumptions can replace these assumptions: Under Assumptions \ref{assumption:indep_arm} and \ref{assumption:selection_possibility}, and $P_i$ in Assumption \ref{assumption:indep_arm} being one of the bases proposed in Section \ref{subsec:bases}, we can conclude $\bar{\phi} > 0$.  Also, as Theorem \ref{theorem:cc_mixture_general} shows, if the arm feature distribution has a mixture component of a greedy method-applicable distribution, then $\bar{\phi} > 0$.

\paragraph{Thresholded lasso bandit~\citep{DBLP:conf/icml/AriuAP22}}
The thresholded lasso bandit estimates the support of $\beta^{*}$ each round and a greedy arm selection policy is performed by the inner product of the arm features and the estimated parameter for $\beta^{*}$ on this support. The relaxed symmetry and the balanced covariance are introduced to ensure proper support estimation under the greedy policy. Specifically, these assumptions are again used for Lemma 10 of \citet{DBLP:conf/icml/OhIZ21} in Lemma 5.4. Therefore, as with the greedy algorithm, our proposed classes can be used as the assumptions of the arm feature distribution.

\paragraph{Greedy algorithm for the combinatorial setting}
Our analysis is applicable to the combinatorial bandit setting, where no regret upper bound for the sparsity-agnostic greedy algorithm is still given. 
Here we consider the setting where in each round, $L$ of the $K$ arms are selected and their respective rewards are observed. Suppose that in each round $t$, the selection policy determines a set of arms $\mathcal{I}_t \subset{[K]}$ where $|\mathcal{I}_t| = L$ to be selected.
The reward under selection criteria is given by:
$
    r_{t} 
        := 
        \sum_{a_t \in \mathcal{I}_t} X_{a_t}^{t\top} \beta^{*} + \epsilon_{t}
    \,,
$
whereas the optimal reward $r_t^{*}$ is given under the optimal arm-set selection $\mathcal{I}_t^{*}$.
The empirical and expected Gram-matrices are given by:
\begin{align}
    G_t & := 
        \frac{1}{L t} \sum_{s=1}^{t} 
        \sum_{a_s \in \mathcal{I}_s}
        X_{a_{s}}^{s} X_{a_{s}}^{s\top}
    \,,
    \nonumber
    \\
    \bar{G}_t & := 
        \frac{1}{Lt} \sum_{s=1}^{t} 
        \mathbb{E} \left[
            \sum_{a_s \in \mathcal{I}_s}
            X_{a_{s}}^{s} X_{a_{s}}^{s\top} \mid \mathcal{F}'_{s-1}
        \right]
    \,,
\end{align}
respectively.
Then, Lemma \ref{lemma:epsX}, \ref{lemma:cc_btwn_empirical_expected} and  \ref{lemma:lasso_estimator} do not depend on the arm selection policy and hold for the combinatorial setting by simply replacing $t$ with $Lt$.

The greedy policy for the combinatorial setting is to choose the top-$L$ arms:
\begin{assumption}
    \label{assumption:arm_selection_comb}
    Under given arm features $\mathcal{X}^t = \{X_1^t, \dots, X_K^t \} \in (\mathbb{R}^{d})^{K}$ at round $t$, the arm selection probability for a set of arms $\mathcal{I}_t$ is given by:
    $
        P(\text{Select $\mathcal{I}_t$} \mid \mathcal{X}^t, \beta_{t-1}) 
        \propto
        \prod_{i \in \mathcal{I}_t} 
        \prod_{j \notin \mathcal{I}_t} 
        I(\beta_{t-1}^{\top} X_i^t \geq \beta_{t-1}^{\top} X_j^t)
    $,
    where the parameter $\beta_{t-1} \in \mathbb{R}^d$ is determined from information prior to $\mathcal{X}^t$ and $r^t$.
\end{assumption}
Then, similar to Lemma \ref{lemma:greedy_bound_empirical}, the following statement holds:
\begin{corollary}
    \label{corollary:combinatorial_bound_empirical}
    Under Assumption \ref{assumption:x_beta_bound}, 
    and if $\bar{\phi} := \text{sup}\{\phi \geq 0 \mid \phi_{\mathcal{S}}(\bar{G}_t) \geq \phi, \forall t, a.s.\} > 0$, the expected regret for the greedy algorithm is upper bounded by:
    $ 
    R(T) =
        O \left(
            \frac{1}{\bar{\phi}^2} \sqrt{\log (d T) T} 
        \right)
    $.
\end{corollary}
The positivity of $\bar{\phi}$ is therefore also important under the combinatorial setting. 
In this regard, the following theorem holds.
\begin{theorem}
    \label{theorem:combinatorial}
    Under Assumption 
    \ref{assumption:indep_arm}, 
    \ref{assumption:selection_possibility}, 
    and 
    \ref{assumption:arm_selection_comb}, 
    and 
    if $P_i(X_i)$ is described by $P_{GM}$, $P_{LGM}$, $P_{D}$, or $P_{R}$, 
    then there exists $\phi'_0 > 0$ that satisfies $\phi_{\mathcal{S}} (\bar{G}_t) > \phi'_0$.
\end{theorem}

\section{Discussion and Conclusion}
\label{section:conclusion}

In this paper, the applicability of the greedy algorithm in the sparse linear contextual bandits was considered. 
We showed that the regret upper bound of the greedy algorithm is guaranteed over a wider range of the arm feature distribution than that covered by the previous works.
In addition, we demonstrated that our analytical approach is not restricted to the simple greedy algorithm, but can be applied in a variety of settings.

On the other hand, the analysis still leaves much room for development.
First, the size of $\phi_0$ has not been fully explored, in particular how the shape of the distribution, the dimension of the arm features and the number of arms affect it.\footnote{
    The quantitative lower bound discussion is difficult without further specification of the arm feature distribution. Specifically, $P(\text{Select $i$} \mid X_i, \beta)$, which appears in the lower bound evaluation, has almost no constraints on its shape under our assumptions. Avoiding additional assumptions, we give formulae that can, in principle, compute a lower bound for $\bar{\phi}$ under a given arm feature distribution, which is summarized in the appendix.
}
Secondly, for truncated distributions, although we gave the radial basis, it is conceivable that a more representational basis could exist.
Finally, we assumed that at least one arm is independent of the other arms, as stated in Assumption \ref{assumption:indep_arm}. 
Although this is looser than the assumption of all arms being independent, which is assumed in many analyses, it is an interesting subject for future research on what class of arm feature distributions can be greedy-applicable when all arms are correlated.

\section*{Acknowledgments}
HS was supported by JSPS KAKENHI Grant Numbers JP17K12646, JP21K17708, and JP21H03397, Japan.
TF was supported by JSPS KAKENHI Grant Numbers JP20H05965, JP21K11759, and JP21H03397, Japan.
NK was supported by JSPS KAKENHI Grant Numbers JP22H05001, JP20H05795, and JP21H03397, Japan and JST ERATO Grant Number JPMJER2301, Japan.
KK was supported by JSPS KAKENHI Grant Numbers JP22H05001 and JP20A402, Japan.

\bibliography{ref}

\begin{thebibliography}{27}
\providecommand{\natexlab}[1]{#1}

\bibitem[{Abbasi{-}Yadkori, P{\'{a}}l, and
  Szepesv{\'{a}}ri(2011)}]{DBLP:conf/nips/Abbasi-YadkoriPS11}
Abbasi{-}Yadkori, Y.; P{\'{a}}l, D.; and Szepesv{\'{a}}ri, C. 2011.
\newblock Improved Algorithms for Linear Stochastic Bandits.
\newblock In \emph{NIPS 2011}, 2312--2320.

\bibitem[{Abbasi{-}Yadkori, P{\'{a}}l, and
  Szepesv{\'{a}}ri(2012)}]{DBLP:journals/jmlr/Abbasi-YadkoriPS12}
Abbasi{-}Yadkori, Y.; P{\'{a}}l, D.; and Szepesv{\'{a}}ri, C. 2012.
\newblock Online-to-Confidence-Set Conversions and Application to Sparse
  Stochastic Bandits.
\newblock In \emph{{AISTATS} 2012}, volume~22, 1--9.

\bibitem[{Abe, Biermann, and Long(2003)}]{DBLP:journals/algorithmica/AbeBL03}
Abe, N.; Biermann, A.~W.; and Long, P.~M. 2003.
\newblock Reinforcement Learning with Immediate Rewards and Linear Hypotheses.
\newblock \emph{Algorithmica}, 37(4): 263--293.

\bibitem[{Ariu, Abe, and Prouti{\`{e}}re(2022)}]{DBLP:conf/icml/AriuAP22}
Ariu, K.; Abe, K.; and Prouti{\`{e}}re, A. 2022.
\newblock Thresholded Lasso Bandit.
\newblock In \emph{{ICML} 2022}, volume 162, 878--928.

\bibitem[{Auer(2002)}]{DBLP:journals/jmlr/Auer02}
Auer, P. 2002.
\newblock Using Confidence Bounds for Exploitation-Exploration Trade-offs.
\newblock \emph{{JMLR}}, 3: 397--422.

\bibitem[{Bastani and Bayati(2020)}]{DBLP:journals/ior/BastaniB20}
Bastani, H.; and Bayati, M. 2020.
\newblock Online Decision Making with High-Dimensional Covariates.
\newblock \emph{Oper. Res.}, 68(1): 276--294.

\bibitem[{Bastani, Bayati, and Khosravi(2021)}]{bastani2021mostly}
Bastani, H.; Bayati, M.; and Khosravi, K. 2021.
\newblock Mostly exploration-free algorithms for contextual bandits.
\newblock \emph{Management Science}, 67(3): 1329--1349.

\bibitem[{Bouneffouf and Rish(2019)}]{DBLP:journals/corr/abs-1904-10040}
Bouneffouf, D.; and Rish, I. 2019.
\newblock A Survey on Practical Applications of Multi-Armed and Contextual
  Bandits.
\newblock \emph{arXiv preprint arXiv:1904.10040}.

\bibitem[{B{\"u}hlmann and Van De~Geer(2011)}]{buhlmann2011statistics}
B{\"u}hlmann, P.; and Van De~Geer, S. 2011.
\newblock \emph{Statistics for high-dimensional data: methods, theory and
  applications}.
\newblock Springer Science \& Business Media.

\bibitem[{Carpentier and Munos(2012)}]{DBLP:journals/jmlr/CarpentierM12}
Carpentier, A.; and Munos, R. 2012.
\newblock Bandit Theory meets Compressed Sensing for high dimensional
  Stochastic Linear Bandit.
\newblock In \emph{{AISTATS} 2012}, volume~22, 190--198.

\bibitem[{Cella et~al.(2023)Cella, Lounici, Pacreau, and
  Pontil}]{cella2023multi}
Cella, L.; Lounici, K.; Pacreau, G.; and Pontil, M. 2023.
\newblock Multi-task representation learning with stochastic linear bandits.
\newblock In \emph{{AISTATS 2023}}, 4822--4847.

\bibitem[{Chu et~al.(2011)Chu, Li, Reyzin, and
  Schapire}]{DBLP:journals/jmlr/ChuLRS11}
Chu, W.; Li, L.; Reyzin, L.; and Schapire, R.~E. 2011.
\newblock Contextual Bandits with Linear Payoff Functions.
\newblock In \emph{{AISTATS} 2011}, volume~15, 208--214.

\bibitem[{Dani, Hayes, and Kakade(2008)}]{DBLP:conf/colt/DaniHK08}
Dani, V.; Hayes, T.~P.; and Kakade, S.~M. 2008.
\newblock Stochastic Linear Optimization under Bandit Feedback.
\newblock In \emph{{COLT} 2008}, 355--366.

\bibitem[{Durand et~al.(2018)Durand, Achilleos, Iacovides, Strati, Mitsis, and
  Pineau}]{DBLP:conf/mlhc/DurandAISMP18}
Durand, A.; Achilleos, C.; Iacovides, D.; Strati, K.; Mitsis, G.~D.; and
  Pineau, J. 2018.
\newblock Contextual Bandits for Adapting Treatment in a Mouse Model of de Novo
  Carcinogenesis.
\newblock In \emph{{MLHC} 2018}, volume~85, 67--82.

\bibitem[{Kannan et~al.(2018)Kannan, Morgenstern, Roth, Waggoner, and
  Wu}]{DBLP:conf/nips/KannanMRWW18}
Kannan, S.; Morgenstern, J.; Roth, A.; Waggoner, B.; and Wu, Z.~S. 2018.
\newblock A Smoothed Analysis of the Greedy Algorithm for the Linear Contextual
  Bandit Problem.
\newblock In \emph{{NeurIPS} 2018}, 2231--2241.

\bibitem[{Kim and Paik(2019)}]{DBLP:conf/nips/KimP19}
Kim, G.; and Paik, M.~C. 2019.
\newblock Doubly-Robust Lasso Bandit.
\newblock In \emph{NeurIPS 2019}, 5869--5879.

\bibitem[{Lattimore and Szepesv{\'a}ri(2020)}]{lattimore2020bandit}
Lattimore, T.; and Szepesv{\'a}ri, C. 2020.
\newblock \emph{Bandit algorithms}.
\newblock Cambridge University Press.

\bibitem[{Li et~al.(2010)Li, Chu, Langford, and
  Schapire}]{DBLP:conf/www/LiCLS10}
Li, L.; Chu, W.; Langford, J.; and Schapire, R.~E. 2010.
\newblock A contextual-bandit approach to personalized news article
  recommendation.
\newblock In \emph{{WWW} 2010}, 661--670.

\bibitem[{Oh and Iyengar(2021)}]{DBLP:conf/aaai/OhI21}
Oh, M.; and Iyengar, G. 2021.
\newblock Multinomial Logit Contextual Bandits: Provable Optimality and
  Practicality.
\newblock In \emph{{AAAI} 2021}, 9205--9213.

\bibitem[{Oh, Iyengar, and Zeevi(2021)}]{DBLP:conf/icml/OhIZ21}
Oh, M.; Iyengar, G.; and Zeevi, A. 2021.
\newblock Sparsity-Agnostic Lasso Bandit.
\newblock In \emph{{ICML} 2021}, volume 139, 8271--8280.

\bibitem[{Raghavan et~al.(2018)Raghavan, Slivkins, Vaughan, and
  Wu}]{DBLP:conf/colt/RaghavanSVW18}
Raghavan, M.; Slivkins, A.; Vaughan, J.~W.; and Wu, Z.~S. 2018.
\newblock The Externalities of Exploration and How Data Diversity Helps
  Exploitation.
\newblock In \emph{{COLT} 2018}, volume~75, 1724--1738.

\bibitem[{Rusmevichientong and
  Tsitsiklis(2010)}]{DBLP:journals/mor/RusmevichientongT10}
Rusmevichientong, P.; and Tsitsiklis, J.~N. 2010.
\newblock Linearly Parameterized Bandits.
\newblock \emph{Math. Oper. Res.}, 35(2): 395--411.

\bibitem[{Sivakumar, Wu, and Banerjee(2020)}]{sivakumar2020structured}
Sivakumar, V.; Wu, S.; and Banerjee, A. 2020.
\newblock Structured linear contextual bandits: A sharp and geometric smoothed
  analysis.
\newblock In \emph{{ICML} 2020}, 9026--9035.

\bibitem[{Sivakumar, Zuo, and Banerjee(2022)}]{DBLP:conf/icml/SivakumarZ022}
Sivakumar, V.; Zuo, S.; and Banerjee, A. 2022.
\newblock Smoothed Adversarial Linear Contextual Bandits with Knapsacks.
\newblock In \emph{{ICML} 2022}, volume 162, 20253--20277.

\bibitem[{Tibshirani(1996)}]{tibshirani1996regression}
Tibshirani, R. 1996.
\newblock Regression shrinkage and selection via the lasso.
\newblock \emph{Journal of the Royal Statistical Society: Series B
  (Methodological)}, 58(1): 267--288.

\bibitem[{Wainwright(2019)}]{wainwright2019high}
Wainwright, M.~J. 2019.
\newblock \emph{High-dimensional statistics: A non-asymptotic viewpoint},
  volume~48.
\newblock Cambridge university press.

\bibitem[{Wang, Wei, and Yao(2018)}]{DBLP:conf/icml/WangWY18}
Wang, X.; Wei, M.~M.; and Yao, T. 2018.
\newblock Minimax Concave Penalized Multi-Armed Bandit Model with
  High-Dimensional Convariates.
\newblock In \emph{{ICML} 2018}, volume~80, 5187--5195.

\end{thebibliography}

\appendix
\onecolumn

\section{Related works}
\paragraph{Stochastic linear contextual bandit}
The linear contextual bandit problem was first introduced by \citet{DBLP:journals/algorithmica/AbeBL03} in the setting of a fixed number of arms.  
Later, \citet{DBLP:journals/jmlr/Auer02} proposed LinRel algorithm for the arm selection by upper confidence intervals obtained by SVD. 
Subsequently, \citet{DBLP:conf/www/LiCLS10} proposed LinUCB algorithm that does not require computationally expensive SVD; a theoretical analysis of the LinUCB variants was carried out by \citet{DBLP:journals/jmlr/ChuLRS11} and a regret upper bound of $O(\sqrt{d T \log^3 T})$ was given.

Another line of research is the setting where the arm features are given as a compact subset of $\mathbb{R}^d$, rather than a fixed finite number of arms.
In this setting, arm selection algorithms based on the confidence ellipsoids were developed by
\citet{DBLP:conf/colt/DaniHK08},
\citet{DBLP:journals/mor/RusmevichientongT10},
and 
\citet{DBLP:conf/nips/Abbasi-YadkoriPS11}.
In \citet{DBLP:conf/colt/DaniHK08} and \citet{DBLP:journals/mor/RusmevichientongT10},
their algorithms gave $O(d \sqrt{T \log^3 T})$ regret upper bounds, which was improved by 
\citet{DBLP:conf/nips/Abbasi-YadkoriPS11} 
to $O(d \sqrt{T} \log T)$.

\paragraph{Assumptions on the arm feature distribution}
The relaxed symmetry was introduced by \citet{bastani2021mostly} as a sufficient condition for their covariate diversity condition, which guarantees sample diversity under the greedy algorithm in the dense setting. 
This relaxed symmetry was adopted by \citet{DBLP:conf/icml/OhIZ21} to analyze the performance of the greedy algorithm in a sparse setting, and has subsequently been adopted in various other settings, such as sequential assortment selection~\citep{DBLP:conf/aaai/OhI21} and multitask learning~\citep{cella2023multi}.

The perturbed context setting was employed in the smoothed analysis, in which a random isotropic Gaussian perturbation was added to the arm features.
The analysis for the greedy algorithm was performed by \citet{DBLP:conf/nips/KannanMRWW18}.
Subsequently, analyses have been carried out in various settings, including \citet{DBLP:conf/colt/RaghavanSVW18} for LinUCB, \citet{sivakumar2020structured} for structured unknown parameters and \citet{DBLP:conf/icml/SivakumarZ022} for contextual linear bandit with knapsacks.

\section{Proofs of the lemmas in Section \ref{section:preliminary}}
    We give the proofs of the lemmas in Section \ref{section:preliminary} for completeness, although these are not essentially different from those presented in \citet{DBLP:conf/icml/OhIZ21}.

\subsection{Proof of Lemma \ref{lemma:epsX}}
\begin{proof}
    For the $\sigma$-algebra $
    \mathcal{F}_{t} := \sigma(\mathcal{X}^1, a_1, r_1, \dots, \mathcal{X}^{t-1}, a_{t-1}, r_{t-1}, \mathcal{X}^t, a_t)$,
    $\epsilon_t (X_{a_t}^t)_{i}$
    satisfies
    $
        \mathbb{E}\left[
            \epsilon_t (X_{a_t}^t)_{i}
            \mid
            \mathcal{F}_{t}
        \right]
        =
        (X_{a_t}^t)_{i}
        \mathbb{E}\left[
            \epsilon_t 
            \mid
            \mathcal{F}_{t}
        \right]
        = 0
    $
    for $t \in [T]$ and $i \in [d]$.
    Then, $\{\epsilon_t X_{a_t}^t\}_{t=1}^{T}$ is a martingale difference sequence for the filtration $\{\mathcal{F}_{t}\}_{t=1}^{T}$.
    Using 
    $\|X_{a_t}^{t}\|_{\infty} \leq x_{\mathrm{max}}$ for any $t \in [T]$
    (from Assumption \ref{assumption:x_beta_bound}),
    the following inequality also holds:
    \begin{align}
        \mathbb{E}\left[
            e^{\lambda \epsilon_t (X_{a_t}^t)_{i}}
            \mid \mathcal{F}_{t}
        \right]
        \leq
            e^{\lambda^2 \sigma^2 x_{\mathrm{max}}^2 / 2}
        \,,
    \end{align}
    almost surely for any $\lambda \in \mathbb{R}$, $t \in [T]$, and $i \in [d]$ due to the $\sigma$-sub-Gaussianity of $\epsilon_{t}$. 
    From Theorem 2.19 (b) in \citet{wainwright2019high}, the following inequality holds for any $\delta' >  0$:
    \begin{align}
        P\left[
            \frac{1}{t}
            \left|
                \sum_{s=1}^{t}
                \epsilon_s (X_{a_s}^s)_{i}
            \right|
            \geq \delta'
        \right]
        \leq 
        2 e^{
            -\frac{
                t \delta'^{2}
            }{
                2 x_{\mathrm{max}}^2 \sigma^2
            }
        }
        \,,
    \end{align}
    which implies
    \begin{align}
        P\left[
            \frac{1}{t}
            \left\|
                \sum_{s=1}^{t}
                \epsilon_s X_{a_s}^s
            \right\|_{\infty}
            \leq \delta'
        \right]
        & =
            1 -
            P\left[
                \bigcup_{i=1}^{d}
                \left(
                \frac{1}{t}
                \left|
                    \sum_{s=1}^{t}
                    \epsilon_s 
                    \left(
                    X_{a_s}^s
                    \right)_{i}
                \right|
                \geq \delta'
                \right)
            \right]
    \nonumber
    \\
        & \geq 
            1 - 
            \sum_{i=1}^{d}
            P\left[
                \frac{1}{t}
                \left|
                    \sum_{s=1}^{t}
                    \epsilon_s 
                    \left(
                    X_{a_s}^s
                    \right)_{i}
                \right|
                \geq \delta'
            \right]
    \nonumber
    \\
        & \geq 
            1 - 
            2 d e^{
                -\frac{
                    t \delta'^2
                }{
                    2 x_{\mathrm{max}}^2 \sigma^2
                }
            }
    \,.
    \end{align}
    By setting
    $
        \delta'
        =
        x_{\mathrm{max}} \sigma 
        \sqrt{
            \frac{\delta^2 + 2 \log d}{t}
        }
    $ for $\delta > 0$, we finally obtain:
    \begin{align}
        P\left[
            \frac{1}{t}
            \left\|
                \sum_{s=1}^{t}
                \epsilon_s X_{a_s}^s
            \right\|_{\infty}
            \leq 
            x_{\mathrm{max}} \sigma
            \sqrt{
                \frac{\delta^2 + 2 \log d}{t}
            }
        \right]
        \geq 
            1 - 2 e^{-\delta^2/2}
    \,.
    \end{align}
\end{proof}

\subsection{Proof of Lemma \ref{lemma:cc_btwn_empirical_expected}}
\begin{proof}
    We define
    \begin{align}
        \gamma_{ij}^{t}
        :=
        \frac{1}{2 x_{\mathrm{max}}^2}
        \left(
            X_{a_t}^t X_{a_t}^{t\top} 
            -
            \mathbb{E}[
                X_{a_t}^t X_{a_t}^{t\top} 
                \mid \mathcal{F}'_{t-1}
            ]
        \right)_{ij}
        \,,
    \end{align}
    for $1 \leq i \leq j \leq d$.
    Then
    $\gamma_{ij}^t$ satisfies
    $\mathbb{E}[\gamma_{ij}^t \mid \mathcal{F}'_{t-1}] = 0$, 
    and since $\|X_i\|_{\infty} \leq x_{\mathrm{max}}$ for $i \in [d]$ (Assumption \ref{assumption:x_beta_bound}), 
    $\mathbb{E}[|\gamma_{ij}^t|^m | \mathcal{F}'_{t-1}] \leq 1$ for any integer $m \geq 2$.
    Then, for $\delta > 0$, Lemma 8 in \citet{DBLP:conf/icml/OhIZ21} gives:
    \footnote{
        We have done the union bound  for $1 \leq i \leq j \leq d$ more precisely, so the factor that is $\log(2d^2)$ in \citet{DBLP:conf/icml/OhIZ21} becomes $\log(d(d+1))$.
    }
    \begin{align}
        P \left(
            \max_{1\leq i \leq j \leq d}
            \frac{1}{t}
            \left|
                \sum_{s=1}^{t} \gamma_{ij}^s 
            \right|
            \geq
            \delta^2
            + 
            \sqrt{2} \delta
            +
            \sqrt{\frac{2 \log (d (d+1))}{t}}
            +
            \frac{\log (d (d+1))}{t}
        \right)
        \leq
        e^{-t \delta^2}
        \,.
    \end{align}
    We here note that the left-hand side of the inequality in the probability can be rewritten as follows:
    \begin{align}
        \max_{1\leq i \leq j \leq d}
        \frac{1}{t}
        \left|
            \sum_{s=1}^{t} \gamma_{ij}^s
        \right|
        &
            =
            \max_{1\leq i \leq j \leq d}
            \frac{1}{2 x_{\mathrm{max}}^2}
            \left|
            \frac{1}{t}
            \sum_{s=1}^{t}
            \left(
                X_{a_s}^s X_{a_s}^{s\top} 
                -
                \mathbb{E}[
                    X_{a_s}^s X_{a_s}^{s\top} 
                    | \mathcal{F}'_{s-1}
                ]
            \right)_{ij}
            \right|
        \nonumber
        \\ 
        &
            =
            \frac{1}{2 x_{\mathrm{max}}^2}
            \left\|
            \frac{1}{t}
            \sum_{s=1}^{t}
                \left(
                X_{a_s}^s X_{a_s}^{s\top} 
                -
                \mathbb{E}[
                    X_{a_s}^s X_{a_s}^{s\top} 
                    | \mathcal{F}'_{s-1}
                ]
                \right)
            \right\|_{\infty}
        \nonumber
        \\ 
        &
            =
            \frac{
                \left\|
                    G_t
                    -
                    \bar{G}_t
                \right\|_{\infty}
            }{2 x_{\mathrm{max}}^2}
        \,.
    \end{align}
    Here we define $\|A\|_{\infty} := \max_{i, j \in [d]} |A_{ij}|$ for a matrix $A \in \mathbb{R}^{d \times d}$.
    Taking $t \geq \log (d(d+1)) /\delta^2$, $\delta \leq 2 - \sqrt{2}$, 
    the right-hand side of the inequality in the probability is bounded as:
    \begin{align}
        \delta^2
        +
        \sqrt{2} \delta
        +
        \sqrt{\frac{2 \log (d (d+1))}{t}}
        +
        \frac{\log (d (d+1))}{t}
        & \leq
            2 \delta^2 + 2 \sqrt{2} \delta
        \nonumber
        \\
        & =
            2 (\delta + \sqrt{2}) \delta
        \nonumber
        \\
        & \leq
            4 \delta
        \,,
    \end{align}
    and we obtain:
    \begin{align}
        P \left(
            \frac{
                \|G_t - \bar{G}_t\|_{\infty}
            }{
                2 x_{\mathrm{max}}^2
            }
            \leq
            4 \delta
        \right)
        \geq
        1 - e^{-t \delta^2}
        \,.
    \end{align}
    Corollary 6.8 in \citet{buhlmann2011statistics} states that if $\bar{G}_t$ has a positive compatibility constant $\bar{\phi} > 0$, and the inequality 
    $
        \|G_t - \bar{G}_t\|_{\infty}
        \leq
        \bar{\phi}^2 / 32 s_0
    $
    is satisfied, then
    $G_t$ has a compatibility constant $\phi_t^2 \geq \bar{\phi}^2 / 2$.
    Therefore, 
    by replacing $\delta$ with $\kappa(\bar{\phi})$ that satisfies the following conditions:
    \begin{align}
        \kappa(\bar{\phi}) \leq \min 
        \left(
            2 - \sqrt{2},
            \frac{\bar{\phi}^2}{256 x_{\mathrm{max}}^2 s_0}
        \right)
        \,,
        \hspace{5pt}
        t \geq \frac{\log(d(d+1))}{\kappa(\bar{\phi})^2} 
        \,,
    \end{align}
    we see that the following inequality holds:
    \begin{align}
        \phi_t^2 \geq \frac{\bar{\phi}^2}{2}
        \,,
    \end{align}
    with probability at least $1 - e^{-t \kappa(\bar{\phi})^2}$.
\end{proof}

\subsection{Proof of Lemma \ref{lemma:lasso_estimator}}
\begin{proof}
    From Lemma \ref{lemma:basic_eq}, we have the following basic inequality for LASSO:
    \begin{align}
        \frac{1}{t}\sum_{s=1}^{t} (
            X^{s\top}_{a_s}
            (
            \hat{\beta}_t
            -
            \beta^{*}
            )
        )^2
        + \lambda_t \|\hat{\beta}_t\|_1
        \leq
        \frac{2}{t}\sum_{s=1}^{t} 
        \epsilon_s X^{s\top}_{a_s} 
            (
            \hat{\beta}_t
            -
            \beta^{*}
            )
        +
        \lambda_t \|\beta^{*}\|_1
        \,.
    \end{align}
    Under the condition in which Eq~\eqref{eq:epsX} holds, setting  
    \begin{align}
        \label{eq:lambda_t}
        \lambda_{t} \geq 4 
        x_{\mathrm{max}} \sigma \sqrt{ (\delta^2 + 2 \log d) / t
        }
        \,,
    \end{align}
    derives:
    \begin{align}
        &
            \frac{1}{t}\sum_{s=1}^{t} (
                X^{s\top}_{a_s} 
                (
                \hat{\beta}_t
                -
                \beta^{*}
                )
            )^2
            + \lambda_t \|\hat{\beta}_t\|_1
            \leq
            \frac{2}{t} 
            \left\|
                \sum_{s=1}^{t}
                \epsilon_s
                X^{s\top}_{a_s}
            \right\|_{\infty}
            \|
            \hat{\beta}_t
            -
            \beta^*
            \|_1
            +
            \lambda_t \|\beta^{*}\|_1
        \nonumber
        \\
        \Rightarrow
        \hspace{5pt}
        &
            \frac{1}{t}\sum_{s=1}^{t} (
                X^{s\top}_{a_s} 
                (
                \hat{\beta}_t
                -
                \beta^{*}
                )
            )^2
            + \lambda_t \|\hat{\beta}_t\|_1
            \leq
            2 x_{\mathrm{max}} \sigma \sqrt{
                \frac{\delta^2 + 2 \log d}{t}
            }
            \|
            \hat{\beta}_t
            -
            \beta^*
            \|_1
            +
            \lambda_t \|\beta^{*}\|_1
        \nonumber
        \\
        \Rightarrow
        \hspace{5pt}
        &
            \frac{1}{t}\sum_{s=1}^{t} (
                X^{s\top}_{a_s} 
                (
                \hat{\beta}_t
                -
                \beta^{*}
                )
            )^2
            + \lambda_t \|\hat{\beta}_t\|_1
            \leq
            \frac{\lambda_{t}}{2} 
            \|
            \hat{\beta}_t
            -
            \beta^*
            \|_1
            +
            \lambda_t \|\beta^{*}\|_1
        \nonumber
        \\
        \Leftrightarrow
        \hspace{5pt}
        &
            \frac{1}{t}\sum_{s=1}^{t} (
                X^{s\top}_{a_s} 
                (
                \hat{\beta}_t
                -
                \beta^{*}
                )
            )^2
            +
            \lambda_t \|\hat{\beta}_t\|_{\mathcal{S}^c, 1}
            +
            \lambda_t \|\hat{\beta}_t\|_{\mathcal{S}, 1}
                \leq
                    \frac{\lambda_{t}}{2} 
                    \|
                    \hat{\beta}_t
                    -
                    \beta^*
                    \|_{\mathcal{S}, 1}
                    +
                    \frac{\lambda_{t}}{2} 
                    \|
                    \hat{\beta}_t
                    \|_{\mathcal{S}^c, 1}
                    +
                    \lambda_t \|\beta^{*}\|_{\mathcal{S}, 1}
        \nonumber
        \\
        \Leftrightarrow
        \hspace{5pt}
        &
            \frac{1}{t}\sum_{s=1}^{t} (
                X^{s\top}_{a_s} 
                (
                \hat{\beta}_t
                -
                \beta^{*}
                )
            )^2
            +
            \frac{\lambda_{t}}{2} 
            \|\hat{\beta}_t\|_{\mathcal{S}^c, 1}
            +
            \lambda_t \|\hat{\beta}_t\|_{\mathcal{S}, 1}
                \leq
                    \frac{\lambda_{t}}{2} 
                    \|
                    \hat{\beta}_t
                    -
                    \beta^*
                    \|_{\mathcal{S}, 1}
                    +
                    \lambda_t \|\beta^{*}\|_{\mathcal{S}, 1}
        \nonumber
        \\
        \Leftrightarrow
        \hspace{5pt}
        &
            \frac{1}{t}\sum_{s=1}^{t} (
                X^{s\top}_{a_s} 
                (
                \hat{\beta}_t
                -
                \beta^{*}
                )
            )^2
            +
            \frac{\lambda_{t}}{2} 
            \|\hat{\beta}_t\|_{\mathcal{S}^c, 1}
            +
                \leq
                    \frac{\lambda_{t}}{2} 
                    \|
                    \hat{\beta}_t
                    -
                    \beta^*
                    \|_{\mathcal{S}, 1}
                    +
                    \lambda_t (
                    \|\beta^{*}\|_{\mathcal{S}, 1}
                    -
                    \|\hat{\beta}_t\|_{\mathcal{S}, 1}
                    )
        \nonumber
        \\
        \Rightarrow
        \hspace{5pt}
        &
            \frac{1}{t}\sum_{s=1}^{t} (
                X^{s\top}_{a_s} 
                (
                \hat{\beta}_t
                -
                \beta^{*}
                )
            )^2
            +
            \frac{\lambda_{t}}{2} 
            \|\hat{\beta}_t\|_{\mathcal{S}^c, 1}
                \leq
                    \frac{3 \lambda_{t}}{2} 
                    \|
                    \hat{\beta}_t
                    -
                    \beta^*
                    \|_{\mathcal{S}, 1}
        \nonumber
        \\
        \Leftrightarrow
        \hspace{5pt}
        &
            \frac{2}{t}\sum_{s=1}^{t} (
                X^{s\top}_{a_s} 
                (
                \hat{\beta}_t
                -
                \beta^{*}
                )
            )^2
            +
            \lambda_t \|\hat{\beta}_t\|_{\mathcal{S}^c, 1}
                \leq
                    3 \lambda_{t}
                    \|
                    \hat{\beta}_t
                    -
                    \beta^*
                    \|_{\mathcal{S}, 1}
        \nonumber
        \\
        \label{eq:basic_eq_in_epsX_event}
        \Leftrightarrow
        \hspace{5pt}
        &
            2
            (\hat{\beta}_t^{\top} - \beta^{*\top})
            G_t
            (\hat{\beta}_t - \beta^*)
            +
            \lambda_t \|\hat{\beta}_t\|_{\mathcal{S}^c, 1}
                \leq
                    3 \lambda_{t}
                    \|
                    \hat{\beta}_t
                    -
                    \beta^*
                    \|_{\mathcal{S}, 1}
        \,.
    \end{align}
    where we use the definition of the empirical Gram-matrix: $G_t := \frac{1}{t}\sum_{s=1}^{t} X^s_{a_s} X^{s\top}_{a_s}$ in the last line.
    The last inequality implies that 
    when Eq.~\eqref{eq:epsX} holds, the gap $\hat{\beta}_t - \beta^*$ is an element of $
    \{
        V \in \mathbb{R}^{d}: 
        \|V\|_{\mathcal{S}^c, 1}
        \leq
        3 \|V\|_{\mathcal{S}, 1}
    \}
    $ 
    and 
    therefore, if the compatibility constant of $G_t$ is positive, i.e., $\phi_t > 0$, then, from the compatibility condition, 
    $\| \hat{\beta}_t - \beta^* \|_{\mathcal{S}, 1}$ 
    is upper bounded by 
    $
        (\hat{\beta}_t^{\top} - \beta^{*\top})
        G_t
        (\hat{\beta}_t - \beta^*)
    $.
    Specifically, the following holds:
    \begin{align}
        2
        (\hat{\beta}_t^{\top} - \beta^{*\top})
        G_t
        (\hat{\beta}_t - \beta^*)
        +
        \lambda_t \|
            \hat{\beta}_t - \beta^*
        \|_1
        =
        &
            2
            (\hat{\beta}_t^{\top} - \beta^{*\top})
            G_t
            (\hat{\beta}_t - \beta^*)
            +
            \lambda_t \|
                \beta^*
            \|_{\mathcal{S}^c,1}
            +
            \lambda_t \|
                \hat{\beta}_t - \beta^*
            \|_{\mathcal{S},1}
        \nonumber
        \\
        \leq
        &
            4
            \lambda_t \|
                \hat{\beta}_t - \beta^*
            \|_{\mathcal{S},1}
        \nonumber
        \\
        \leq
        &
            \frac{4 \lambda_t \sqrt{s_0}}{\phi_t}
            \sqrt{
                (\hat{\beta}_t^{\top} - \beta^{*\top})
                G_t
                (\hat{\beta}_t - \beta^*)
            }
        \nonumber
        \\
        \leq
        &
            \frac{ 4 \lambda_t^2 s_0}{\phi_t^2} 
            +
            (\hat{\beta}_t^{\top} - \beta^{*\top})
            G_t
            (\hat{\beta}_t - \beta^*)
        \,,
    \end{align}
    In the second line, we use Eq.~\eqref{eq:basic_eq_in_epsX_event} 
    and in the third line, we use the compatibility condition.
    Under the assumption where Eq.~\eqref{eq:cc_btwn_empirical_expected} holds, we obtain
    \begin{align}
        (\hat{\beta}_t^{\top} - \beta^{*\top})
        G_t
        (\hat{\beta}_t - \beta^*)
        +
        \lambda_t \|
            \hat{\beta}_t - \beta^*
        \|_1
        \leq
            \frac{ 4 \lambda_t^2 s_0}{\phi_t^2}
        \leq
            \frac{ 8 \lambda_t^2 s_0}{\bar{\phi}^2}
        \,.
    \end{align}
    The above inequality implies
    that 
    under the conditions Eq.~\eqref{eq:epsX} and Eq.~\eqref{eq:cc_btwn_empirical_expected}, 
    and $\lambda_t$ satisfying Eq.~\eqref{eq:lambda_t}, the following inequality holds:
    \begin{align}
        \|
            \beta^* - \hat{\beta}_t
        \|_1
        \leq
        &
            \frac{ 8 \lambda_t s_0}{\bar{\phi}^2}
        \,,
    \end{align}
    which is the claim of the lemma.
\end{proof}

\begin{lemma}[Lemma 6.1 in \citet{buhlmann2011statistics}]

\label{lemma:basic_eq}
    Under the reward model of 
    $
        r_{t} := 
            X_{a_t}^{t\top} \beta^*
            + \epsilon_t
    $, and LASSO estimator 
    $
        \hat{\beta}_{t} = 
        \mathrm{argmin}_{\beta} \frac{1}{t}\sum_{s=1}^{t} (r_s - X^{s\top}_{a_s} \beta)^2 + \lambda_t \|\beta\|_1
    $,
    the following basic inequality holds:
    \begin{align}
        \label{eq:basic_eq}
        \frac{1}{t}\sum_{s=1}^{t} (
            X^{s\top}_{a_s} 
            (
            \hat{\beta}_t
            -
            \beta^{*}
            )
        )^2
        + \lambda_t \|\hat{\beta}_t\|_1
        \leq
        \frac{2}{t}\sum_{s=1}^{t} 
        \epsilon_s X^{s\top}_{a_s} 
            (
            \hat{\beta}_t
            -
            \beta^{*}
            )
        +
        \lambda_t \|\beta^{*}\|_1
        \,.
    \end{align}
\end{lemma}
\begin{proof}
    From the definition, the following inequality holds:
    \begin{align}
        &
        \frac{1}{t}\sum_{s=1}^{t} (r_s - X^{s\top}_{a_s} \hat{\beta}_t)^2 + \lambda_t \|\hat{\beta}_t\|_1
        \leq
        \frac{1}{t}\sum_{s=1}^{t} (r_s - X^{s\top}_{a_s} \beta^{*})^2
        + \lambda_t \|\beta^{*}\|_1
        \,, 
        \nonumber
        \\
        \Rightarrow
        \hspace{5pt}
        &
            \frac{1}{t}\sum_{s=1}^{t} (
                X^{s\top}_{a_s} \beta^{*}
                - X^{s\top}_{a_s} \hat{\beta}_t
                + \epsilon_s
            )^2 
            + \lambda_t \|\hat{\beta}_t\|_1
            \leq
            \frac{1}{t}\sum_{s=1}^{t} (
                \epsilon_s
            )^2
            + \lambda_t \|\beta^{*}\|_1
            \,, 
        \nonumber
        \\
        \Rightarrow
        \hspace{5pt}
            &
            \frac{1}{t}\sum_{s=1}^{t} (
                X^{s\top}_{a_s} 
                (\beta^{*} - \hat{\beta}_t)
            )^2
            +
            \frac{2}{t}\sum_{s=1}^{t} 
            \epsilon_s X^{s\top}_{a_s} (
                \beta^{*} -\hat{\beta}_t
            )
            + \lambda_t \|\hat{\beta}_t\|_1
            \leq
            \lambda_t \|\beta^{*}\|_1
        \,.
    \end{align}
    The last inequality derives Eq.~\eqref{eq:basic_eq}.
\end{proof}

\subsection{Proof of Lemma \ref{lemma:greedy_bound_empirical}}
\begin{proof}
    We define the following events at round $t$:
    \begin{align}
        \mathcal{E}_{\mathrm{gap}}^t
        = \{
            \|\beta^* - \hat{\beta}_t \|_1 \leq 
            8\lambda_t s_0 / \bar{\phi}^2
        \}
        \,,
    \end{align}
    and below, we set $\lambda_{t} = 4 x_{\mathrm{max}} \sigma \sqrt{ (\delta^2 + 2 \log d) / t}$.
    By Lemma \ref{lemma:lasso_estimator}, the above inequality is valid when both events Eq.~\eqref{eq:epsX} and Eq.~\eqref{eq:cc_btwn_empirical_expected} occur simultaneously.
    From Lemmas \ref{lemma:epsX} and \ref{lemma:cc_btwn_empirical_expected}, the probability where the event $\mathcal{E}_{\mathrm{gap}}^t$ holds is given by:
    \begin{align}
        P(\mathcal{E}_{\mathrm{gap}}^t)
        \geq
        1 - 2 e^{-\delta^2/2} - e^{-t \kappa(\bar{\phi})^2}\,,
    \end{align}
    for
    $t \geq T_0 := \log(d (d+1))/\kappa(\bar{\phi})$,
    where 
    $\kappa(\bar{\phi}) := \min (2 - \sqrt{2}, \bar{\phi}^2 / (256 x_{\mathrm{max}}^2 s_0))$.
    We define the regret at round $t$ as follows:
    $
        \mathrm{reg} (t) : 
        =
            X_{a_t^*}^{t\top} \beta^* 
            - X_{a_t}^{t\top} \beta^*
    $.
    Below, we use $\mathrm{reg}(t) \leq 2 x_{\mathrm{max}} b$, and the following inequality for the greedy algorithm:
    \begin{align}
        \mathrm{reg}(t) 
        & = 
            X_{a_t^*}^{t\top} \beta^* 
            - X_{a_t}^{t\top} \beta^*
        \nonumber
        \\
        & = 
            X_{a_t^*}^{t\top} \beta^* 
            - X_{a_t^*}^{t\top} \hat{\beta}_{t-1}
            + X_{a_t^*}^{t\top} \hat{\beta}_{t-1}
            - X_{a_t}^{t\top} \beta^*
        \nonumber
        \\
        & \leq
            X_{a_t^*}^{t\top} \beta^* 
            - X_{a_t^*}^{t\top} \hat{\beta}_{t-1}
            + X_{a_t}^{t\top} \hat{\beta}_{t-1}
            - X_{a_t}^{t\top} \beta^*
        \nonumber
        \\
        & =
            X_{a_t^*}^{t\top} (
                \beta^* - \hat{\beta}_{t-1})
            + 
            X_{a_t}^{t\top} (
                \hat{\beta}_{t-1} - \beta^*
            )
        \nonumber
        \\
        & =
            (X_{a_t^*}^{t\top} - X_{a_t}^{t\top})
            (\beta^* - \hat{\beta}_{t-1})
        \nonumber
        \\
        & \leq 
            \|X_{a_t^*}^t - X_{a_t}^t\|_{\infty}
            \|\beta^* - \hat{\beta}_{t-1}\|_1
        \nonumber
        \\
        \label{eq:reward_beta_gap}
        & \leq 
            2 x_{\mathrm{max}}
            \|\beta^* - \hat{\beta}_{t-1}\|_1
        \,.
    \end{align}
    The expected reward at round $t$ can be decomposed of 
    \begin{align}
        R(T) 
        & = 
            \sum_{t=1}^{T} \mathbb{E}[
                \mathrm{reg}(t)]
        \nonumber
        \\
        & = 
            \sum_{t=1}^{T_0} \mathbb{E}[
                \mathrm{reg}(t)]
            +
            \sum_{t=T_0+1}^{T} 
            \mathbb{E}[
                \mathrm{reg}(t)
                (
                I[ 
                    \mathcal{E}_{\mathrm{gap}}^{t-1\,c}
                ]
                +
                I[
                    \mathcal{E}_{\mathrm{gap}}^{t-1}
                ]
                )
            ]
        \nonumber
        \\
        & \leq 
            2 x_{\mathrm{max}} b T_0
            +
            \sum_{t=T_0+1}^{T}
            2 x_{\mathrm{max}} b
            P(
                \mathcal{E}_{\mathrm{gap}}^{t-1\,c}
            )
            +
            \mathbb{E}[
                \mathrm{reg}(t)
                I[ \mathcal{E}_{\mathrm{gap}}^{t-1}]
            ]
        \nonumber
        \\
        & \leq 
            2 x_{\mathrm{max}} b T_0
            +
            2 x_{\mathrm{max}} 
            \left(
            \sum_{t=T_0+1}^{T}
            b
            P(
                \mathcal{E}_{\mathrm{gap}}^{t-1\,c}
            )
            +
            \mathbb{E}[
                \|\beta^{*} - \hat{\beta}_{t-1}\|_1
                I[\mathcal{E}_{\mathrm{gap}}^{t-1}]
            ]
            \right)
        \nonumber
        \\
        & =
            2 x_{\mathrm{max}} b T_0
            +
            2 x_{\mathrm{max}} 
            \left(
            \sum_{t=T_0}^{T-1}
            b
            P(
                \mathcal{E}_{\mathrm{gap}}^{t\,c}
            )
            +
            \mathbb{E}[
                \|\beta^{*} - \hat{\beta}_{t}\|_1
                I[\mathcal{E}_{\mathrm{gap}}^{t}]
            ]
            \right)
        \nonumber
        \\
        & \leq
            2 x_{\mathrm{max}} b T_0
            +
            2 x_{\mathrm{max}} 
            \left(
            \sum_{t=T_0}^{T}
            b
            P(
                \mathcal{E}_{\mathrm{gap}}^{t\,c}
            )
            +
            \mathbb{E}[
                \|\beta^{*} - \hat{\beta}_{t}\|_1
                I[\mathcal{E}_{\mathrm{gap}}^{t}]
            ]
            \right)
        \,.
    \end{align}
    In the fourth line, we use Eq.~\eqref{eq:reward_beta_gap}.
    Applying the inequality given by the event $\mathcal{E}_{\mathrm{gap}}^{t}$ to the last term above, we obtain the following bound:
    \begin{align}
        R(T)
        & \leq 
            2 x_{\mathrm{max}} b T_0
            +
            2 x_{\mathrm{max}} 
            \left(
            \sum_{t=T_0}^{T}
            b
            P(
                \mathcal{E}_{\mathrm{gap}}^{t\,c}
            )
            +
            \frac{8 \lambda_t s_0}{\bar{\phi}^2}
            \mathbb{E}[
                I[\mathcal{E}_{\mathrm{gap}}^t]
            ]
            \right)
        \nonumber
        \\
        & \leq 
            2 x_{\mathrm{max}} b T_0
            +
            2 x_{\mathrm{max}} 
            \left(
            \sum_{t=T_0}^{T}
            b
            P(
                \mathcal{E}_{\mathrm{gap}}^{t\,c}
            )
            +
            \frac{8 \lambda_t s_0}{\bar{\phi}^2}
            \right)
        \nonumber
        \\
        & \leq
            2 x_{\mathrm{max}} b T_0
            +
            2 x_{\mathrm{max}} 
            \left(
            \sum_{t=T_0}^{T}
            b
            e^{-t \kappa(\bar{\phi})^2}
            +
            2 
            b
            e^{-\delta^2/2}
            +
            \frac{
                32 s_0
                x_{\mathrm{max}} \sigma
            }{\bar{\phi}^2}
            \sqrt{\frac{\delta^2 + 2 \log d}{t}}
            \right)
        \,,
    \end{align}
    where we substitute $\lambda_t = 4 x_{\mathrm{max}} \sigma \sqrt{(\delta^2 + 2 \log d)/ t}$ in the last equality.
    Taking $\delta^2 = 4 \log t$, we obtain:
    \begin{align}
        R(T) \leq
            2 x_{\mathrm{max}} b T_0
            +
            2 x_{\mathrm{max}} 
            \left(
            \sum_{t=T_0}^{T}
            b
            e^{-t \kappa(\bar{\phi})^2}
            +
            \frac{2 b}{t^2}
            +
            \frac{
                32 s_0
                x_{\mathrm{max}} \sigma
            }{\bar{\phi}^2}
            \sqrt{\frac{4 \log t + 2 \log d}{t}}
            \right)
        \,.
    \end{align}
    The second, third and fourth terms are upper bounded by:
    \begin{align}
        &
        \sum_{t=T_0}^T b e^{-t \kappa(\bar{\phi})^2} 
            \leq
            \int_{0}^{\infty}
            b e^{-t\kappa(\bar{\phi})^2}
            dt
            =
            \frac{b}{\kappa(\bar{\phi})^2}
        \,,
        \nonumber
        \\ 
        &
        \sum_{t=T_0}^T
        \frac{2 b}{t^2}
            \leq
            \sum_{t=1}^{\infty}
            \frac{2 b}{t^2}
            =
            \frac{\pi^2 b}{3}
        \,,
        \nonumber
        \\ 
        &
        \sum_{t=T_0}^T
        \sqrt{\frac{4\log t + 2\log d}{t}}
        \leq
        \int_{0}^{T}
        \sqrt{\frac{4\log T + 2\log d}{t}}
        dt
        =
        2 \sqrt{(4\log T + 2\log d)T}
        \,,
    \end{align}
    respectively. Here, we use $\sum_{t=1}^{\infty} 1/t^2 = \pi^2/6$ in the second line. Substituting $T_0 = \log (d (d+1)) / \kappa(\bar{\phi})^2$, we finally obtain the following upper bound:
    \begin{align}
        R(T) 
        & \leq
            2 x_{\mathrm{max}} b
            \left(
            \frac{1 + \log (d (d+1))}{ \kappa(\bar{\phi})^2}
            +
            \frac{\pi^2}{3}
            \right)
            +
            128 s_0
            x_{\mathrm{max}}^2 \sigma
            \frac{\sqrt{(4\log T + 2\log d) T}}{\bar{\phi}^2}
        \,.
    \end{align}
\end{proof}

\section{
Proofs of the positivity of $\bar{\phi}$ under Assumptions \ref{assumption:CC}, \ref{assumption:RS} and \ref{assumption:BC}
}

In Section \ref{subsec:assumptions}, we present Assumptions \ref{assumption:CC}, \ref{assumption:RS}, and \ref{assumption:BC} employed in the existing studies for $\bar{\phi} > 0$. In this section, we give proofs of $\bar{\phi} > 0$ under these assumptions.

\begin{lemma}[Lemma 2 in \citet{DBLP:conf/icml/OhIZ21}]
    For $K = 2$, under Assumption \ref{assumption:CC} and \ref{assumption:RS}, the greedy arm selection policy satisfies $\bar{\phi} > 0$.
\end{lemma}
\begin{proof}
    Under the greedy arm selection policy, we have:
    \begin{align}
        \bar{G}_t
        & =
            \frac{1}{t}
            \sum_{s=1}^{t}
            \mathbb{E}
            \left[
                X_{a_s}^{s}
                X_{a_s}^{s\top}
                \mid
                \mathcal{F}'_{s-1}
            \right]
        \nonumber
        \\
        & =
            \frac{1}{t}
            \sum_{s=1}^{t}
            \int
                X_{1} X_{1}^{\top} 
                I \left[
                    \beta_{s-1}^{\top} X_1 \geq \beta_{s-1}^{\top} X_2
                \right]
                P (\mathcal{X})
            dX_{1} dX_{2}
            \nonumber \\
            &
            \hspace{150pt}
            +
            \int
                X_{2} X_{2}^{\top} 
                I \left[
                    \beta_{s-1}^{\top} X_2 \geq \beta_{s-1}^{\top} X_1
                \right]
            P (\mathcal{X})
            dX_{1} dX_{2}
            \,.
    \label{eq:intermediate_K2}
    \end{align}
    The first term in Eq.~\eqref{eq:intermediate_K2} can be bounded as:
    \begin{align}
        \int
            X_{1} X_{1}^{\top} 
            I \left[
                \beta_{s-1}^{\top} X_1 \geq \beta_{s-1}^{\top} X_2
            \right]
            P (\mathcal{X})
        dX_{1} dX_{2}
        \nonumber \\ 
        &
        \hspace{-175pt}
        =
        \left(
            \frac{1}{1 + \nu}
            +
            \frac{\nu}{1 + \nu}
        \right)
        \int
            X_{1} X_{1}^{\top} 
            I \left[
                \beta_{s-1}^{\top} X_1 \geq \beta_{s-1}^{\top} X_2
            \right]
            P (\mathcal{X})
        dX_{1} dX_{2}
        \nonumber \\ 
        &
        \hspace{-175pt}
        =
        \frac{1}{1 + \nu}
        \int
            X_{1} X_{1}^{\top} 
            I \left[
                \beta_{s-1}^{\top} X_1 \geq \beta_{s-1}^{\top} X_2
            \right]
            P (\mathcal{X})
        dX_{1} dX_{2}
        +
        \frac{\nu}{1 + \nu}
        \int
            X_{1} X_{1}^{\top} 
            I \left[
                -\beta_{s-1}^{\top} X_1 \geq -\beta_{s-1}^{\top} X_2
            \right]
            P (\mathcal{-X})
        dX_{1} dX_{2}
        \nonumber \\ 
        &
        \hspace{-175pt}
        \succeq
        \frac{1}{1 + \nu}
        \int
            X_{1} X_{1}^{\top} 
            I \left[
                \beta_{s-1}^{\top} X_1 \geq \beta_{s-1}^{\top} X_2
            \right]
            P (\mathcal{X})
        dX_{1} dX_{2}
        +
        \frac{1}{1 + \nu}
        \int
            X_{1} X_{1}^{\top} 
            I \left[
                \beta_{s-1}^{\top} X_2 \geq \beta_{s-1}^{\top} X_1
            \right]
            P (\mathcal{X})
        dX_{1} dX_{2}
        \nonumber \\
        &
        \hspace{-175pt}
        =
        \frac{1}{1 + \nu}
        \int
            X_{1} X_{1}^{\top} 
            P (\mathcal{X})
        dX_{1} dX_{2}
        \,.
    \end{align}
    In the second line, we transform $\{X_1, X_2\} \rightarrow \{-X_1, -X_2\}$ and in the third line, we use Assumption \ref{assumption:RS}.
    A similar bound holds for the second term in Eq.~\eqref{eq:intermediate_K2}, and we obtain:
    \begin{align}
        \bar{G}_t
        \succeq
            \frac{1}{t}
            \sum_{s=1}^{t}
            \frac{1}{1 + \nu}
            \int
                (X_{1} X_{1}^{\top} 
                +
                X_{2} X_{2}^{\top})
                P (\mathcal{X})
            dX_{1} dX_{2}
        =
        \frac{1}{1 + \nu}
        \mathbb{E}
        \left[
            \sum_{i=1}^{2}
            X_i X_i^{\top}
        \right]
        \,.
    \end{align}
    Combining Assumption \ref{assumption:CC}, we obtain 
    $
        \phi_{\mathcal{S}}(
            \bar{G}_t
        ) \geq 2 \phi_0 / (1 + \nu) > 0
    $.
\end{proof}

\begin{lemma}[Lemma 10 in \citet{DBLP:conf/icml/OhIZ21}]
    For $K > 2$, under Assumptions \ref{assumption:CC}, \ref{assumption:RS}, and \ref{assumption:BC}, the greedy arm selection policy satisfies $\bar{\phi} > 0$.
\end{lemma}
\begin{proof}
    Below, we define the symmetric group for $[K]$ by $\text{Sym}([K])$ and $\sum_{\text{Sym}([K])}$ represents the sum of the overall permutations $\{i_1, \dots, i_{K} \} \in \text{Sym} ([K])$.
    Using the equality $
    1 = 
            \sum_{\text{Sym}([K])}
                I \left[
                    \beta_{s-1}^{\top} X_{i_1} 
                    \geq 
                    \beta_{s-1}^{\top} X_{i_2}
                    \geq
                    \dots
                    \geq
                    \beta_{s-1}^{\top} X_{i_{K}}
                \right]
    $, 
    for all $j \in [K]$, we have:
    \begin{align}
        \int
            X_{j} X_{j}^{\top} 
            P (\mathcal{X})
        \prod_{k'=1}^{K}
        dX_{k'}
        &
            =
            \sum_{\text{Sym}([K])}
            \int
                X_{j} X_{j}^{\top} 
                I \left[
                    \beta_{s-1}^{\top} X_{i_1} 
                    \geq 
                    \beta_{s-1}^{\top} X_{i_2}
                    \geq
                    \dots
                    \geq
                    \beta_{s-1}^{\top} X_{i_{K}}
                \right]
                P (\mathcal{X})
            \prod_{k'=1}^{K}
            dX_{k'}
        \nonumber \\
        & \hspace{-50pt}
            \preceq
            \sum_{\text{Sym}([K])}
            \int
                C'_{\text{BC}}
                \left(
                    X_{i_1} X_{i_1}^{\top} 
                    +
                    X_{i_{K}} X_{i_{K}}^{\top} 
                \right)
                I \left[
                    \beta_{s-1}^{\top} X_{i_1} 
                    \geq 
                    \beta_{s-1}^{\top} X_{i_2}
                    \geq
                    \dots
                    \geq
                    \beta_{s-1}^{\top} X_{i_{K}}
                \right]
                P (\mathcal{X})
            \prod_{k'=1}^{K}
            dX_{k'}
        \nonumber \\ 
        & \hspace{-50pt}
            \preceq
            \sum_{\text{Sym}([K])}
            \int
                C'_{\text{BC}}
                (1 + \nu)
                X_{i_1} X_{i_1}^{\top} 
                I \left[
                    \beta_{s-1}^{\top} X_{i_1} 
                    \geq 
                    \beta_{s-1}^{\top} X_{i_2}
                    \geq
                    \dots
                    \geq
                    \beta_{s-1}^{\top} X_{i_{K}}
                \right]
                P (\mathcal{X})
            \prod_{k'=1}^{K}
            dX_{k'}
        \nonumber \\ 
        & \hspace{-50pt}
            =
            C'_{\text{BC}}
            (1 + \nu)
            \mathbb{E}
            \left[
                X_{a_s}^{s}
                X_{a_s}^{s\top}
                \mid
                \mathcal{F}'_{s-1}
            \right]
        \,.
    \end{align}
    Here we define $C'_{\text{BC}} := \max(C_{\text{BC}}, 1)$ and we use Assumptions \ref{assumption:BC} and \ref{assumption:RS} for the first and second inequalities, respectively.
    Therefore, the compatibility condition is bounded as:
    \begin{align}
        \phi_{\mathcal{S}} (\bar{G}_t)
        =
        \phi_{\mathcal{S}} \left(
            \frac{1}{t}
            \sum_{s=1}^{t}
            \mathbb{E}
            \left[
                X_{a_s}^{s}
                X_{a_s}^{s\top}
                \mid
                \mathcal{F}'_{s-1}
            \right]
        \right)
        \geq
            \frac{1}
            {
                C'_{\text{BC}}
                (1 + \nu)
            }
            \phi_{\mathcal{S}} \left(
                \frac{1}{K}
                \sum_{k=1}^{K}
                \mathbb{E} \left[
                    X_{k} X_{k}^{\top}
                \right]
            \right)
            = 
            \frac{\phi_0}{
                C'_{\text{BC}}
                (1 + \nu)
            }
        \,,
    \end{align}
    which concludes the proof.
\end{proof}

\section{Proofs of the results in Section \ref{section:extension}}
Before the proof, we define the set of functions as follows:
\begin{align}
    \mathcal{F}_d := \{f: \mathbb{R}^d \rightarrow \mathbb{R}_{\ge 0} \mid f: \text{measurable}, 
    \exists \epsilon > 0, 
    \{x \in \mathbb{R}^d \mid f(x) \ge \epsilon \}
    \,\text{has non-zero measure} \}
    \,.
    \label{eq:function_set}
\end{align}
In addition, we give a basic property about $\phi_{\mathcal{S}}$ as the following lemma:
\begin{lemma}
    \label{lemma:cc_ineq}
    Let $\Sigma_n \in \mathbb{R}^{d\times d}$ be a semi-definite matrix for $n \in [N]$.
    A weighted sum $\sum_{n=1}^{N} w_n \Sigma_n$ where $w_n > 0$ for $n \in [N]$ satisfies
    $
        \phi_{\mathcal{S}} (
            \sum_{n=1}^{N} w_n 
            \Sigma_n
        ) 
        \geq
        \sum_{n=1}^{N} w_n 
        \phi_{\mathcal{S}} (
            \Sigma_n
        ) 
    $.
\end{lemma}
\begin{proof}
    Defining 
    $
        \mathcal{D} := 
            \{
                V \in \mathbb{R}^{d}
                \mid
                \| V \|_{\mathcal{S}^c, 1} 
                \leq
                3 \| V \|_{\mathcal{S}, 1} 
            \}
    $,
    we see
    \begin{align}
        \min_{V \in \mathcal{D}}
        \left(
            \frac{
                V^{\top}
                    \sum_{n=1}^{N} w_n 
                    \Sigma_n
                V
            }
            {
                \|V\|_{\mathcal{S}, 1}^2
            }
        \right)
        \geq
        \sum_{n=1}^{N} w_n 
        \left(
            \min_{V_n \in \mathcal{D}}
            \frac{
                V_n^{\top}
                    \Sigma_n
                V_n
            }
            {
                \|V_n\|_{\mathcal{S}, 1}^2
            }
        \right)
        \,,
    \end{align}
    which concludes the proof.
\end{proof}

\subsection{Proof of Theorem \ref{theorem:cc_mixture_general}}
\label{proof:cc_mixture_general}
\begin{proof}
    From the definition, we have
    \begin{align}
        \bar{G}_t (P)
        & =
            \frac{1}{t}
            \sum_{s=1}^{t}
            \mathbb{E}
            \left[
                X_{a_s}^{s}
                X_{a_s}^{s\top}
                \mid
                \mathcal{F}'_{s-1}
            \right]
        \nonumber
        \\
        & =
            \frac{1}{t}
            \sum_{s=1}^{t}
            \sum_{k=1}^{K}
            \int
            X_k X_k^{\top} 
            P(\text{Select $k$} \mid \mathcal{X}, \hat{\beta}_{s-1})
            P (\mathcal{X})
            \prod_{k'=1}^{K} dX_{k'}
        \nonumber
        \\
        & =
            \frac{1}{t}
            \sum_{s=1}^{t}
            \sum_{k=1}^{K}
            \int
            X_k X_k^{\top} 
            P(\text{Select $k$} \mid \mathcal{X}, \hat{\beta}_{s-1})
            (
                c \tilde{P}(\mathcal{X})
                +
                (1 - c) Q (\mathcal{X})
            )
            \prod_{k'=1}^{K} dX_{k'}
        \nonumber
        \\
        & \succeq 
            c
            \frac{1}{t}
            \sum_{s=1}^{t}
            \sum_{k=1}^{K}
            \int
            X_k X_k^{\top} 
            P(\text{Select $k$} \mid \mathcal{X}, \hat{\beta}_{s-1})
            \tilde{P}(\mathcal{X})
            \prod_{k'=1}^{K} dX_{k'}
        \nonumber
        \\
        & =
            c
            \bar{G}_t (\tilde{P})
        \,.
    \end{align}
    By the definition of $\tilde{P}(\mathcal{X})$, we see $\phi_{\mathcal{S}}(\bar{G}_t(P)) \geq c \phi_{\mathcal{S}}(\bar{G}_t(\tilde{P})) = c \phi_0$ from the last line.
\end{proof}

\subsection{Proof of Lemma \ref{lemma:Gt}}
\begin{proof}
    \label{proof:Gt}
    We consider the compatibility constant for 
    $\mathbb{E}[X_{a_t}^t X_{a_t}^{t\top} | \mathcal{F}'_{t-1}]$. From Assumption \ref{assumption:arm_selection}, the expectation is given by
    \begin{align}
    \mathbb{E}[X_{a_t}^t X_{a_t}^{t\top} | \mathcal{F}'_{t-1}]
    & =
        \sum_{k} \int 
        X_{k} X_{k}^{\top} 
        P(
            \text{Select $k$}
            \mid
            \mathcal{X},
            \hat{\beta}_{t-1}
        )
        P(\mathcal{X})
        \prod_{k'=1}^{K} 
        dX_{k'}
    \,,
    \nonumber
    \\
    & =
        \tilde{G}_{\hat{\beta}_{t-1}}
    \,.
    \end{align}
    Therefore,
    using Lemma \ref{lemma:cc_ineq},
    $
        \phi_{\mathcal{S}} (\bar{G}_t) 
        =
        \phi_{\mathcal{S}}(
            \frac{1}{t} \sum_{s=1}^{t} 
            \tilde{G}_{\hat{\beta}_{s-1}}
        ) 
        \geq
        \frac{1}{t} \sum_{s=1}^{t} 
        \phi_{\mathcal{S}}(\tilde{G}_{\hat{\beta}_{s-1}}) 
        > 
        \phi_0
    $.
\end{proof}

\subsection{Proof of Lemma \ref{lemma:gm_lower_bound}}
The outline of the proof is as follows:
Under Assumptions \ref{assumption:arm_selection} and \ref{assumption:indep_arm}, $\tilde{G}_{\beta}$ is lower-bounded by an integral involving $P_{i}(X_i) = P_{\text{GM}}(X_i)$ and $f_{\beta} = P(\text{Select $i$} \mid X_i, \beta)$ (Lemma \ref{lemma:tilde_G}), and under Assumption \ref{assumption:selection_possibility}, we see $f_{\beta} \in \mathcal{F}_{1}$ defined in Eq.~\eqref{eq:function_set}.
Under $f_{\beta} \in \mathcal{F}_1$, for the Gaussian mixture basis $P_{\text{GM}}$, we have Eq.~\eqref{eq:tilde_G_lower_GM} by Lemma \ref{lemma:gauss_second_moment_ineq}, in which the coefficient $c_n(\beta)$ is shown to be positive by Lemma \ref{lemma:c_positive}.

\begin{proof}
    \label{proof:gm_lower_bound}
    Under Assumptions \ref{assumption:arm_selection} and
    \ref{assumption:indep_arm}, 
    Lemma \ref{lemma:tilde_G} gives:
    \begin{align}
        \tilde{G}_{\beta}
        \succeq
            \int
            R_{\beta}^{\top}
            Z Z^{\top}
            R_{\beta}
            f_{\beta}(z_1)
            P_{GM}(R_{\beta}^{\top} Z)
            dZ
        \,,
        \label{eq:tilde_G_lower_GM}
    \end{align}
    where 
    $Z = (z_1, \dots, z_d) \in \mathbb{R}^d$,
    $f_{\beta}: \mathbb{R} \rightarrow \mathbb{R}$ is a non-negative Lebesgue integrable function defined by Eq.~\eqref{eq:P_to_f},
    and
    $R_{\beta} \in \mathbb{R}^{d \times d}$ is an orthogonal matrix that satisfies $R_{\beta} \beta = (\| \beta \|_2, 0, \dots, 0)^{\top}$. 
    On the other hand, from Assumption \ref{assumption:selection_possibility}, we have
    \begin{align}
        P(\text{Select $i$} \mid \beta)
        & =
            \int
            P(\text{Select $i$} \mid X_i, \beta)
            P_{GM}(X_i)
            d X_i
        \nonumber
        \\
        & =
            \int
            f_{\beta} ((X_i)_{\beta})
            P_{GM}(X_i)
            d X_i
        \nonumber
        \\
        & =
            \int
            f_{\beta} (z_1)
            P_{GM}(R_{\beta}^{\top} Z)
            d Z 
        > 0
        \,,
        \label{eq:selection_possibility}
    \end{align}
    for any $\beta$.
    Therefore, 
    $
        f_{\beta} (z_1)
        P_{GM}(R_{\beta}^{\top} Z)
    $
    has a non-empty support.
    Because $P_{GM} (X) > 0$ for any $X \in \mathbb{R}^d$, 
    the condition implies 
    that 
    $ f_{\beta} (z)$ has a non-empty support $\mathcal{R}_{\beta} \subset \mathbb{R}$.
    In addition, by definition, if $f_{\beta}(z) > 0$ for some $z \in \mathbb{R}$, then $f_{\beta}(z') > 0$ for $z' > z$, and therefore 
    for any $\beta$,
    $f_{\beta}$ is an element of $\mathcal{F}_1$ defined in Eq.~\eqref{eq:function_set}.

    In addition, $P_{GM}(R_{\beta}^{\top} Z)$ can be written as:
    \begin{align}
        P_{GM} (R_{\beta}^{\top} Z) 
        = 
            \sum_{n=1}^{N} w_{n}
            \mathcal{N}(R_{\beta}^{\top} Z|\mu_{n}, \Sigma_{n})
        = 
            \sum_{n=1}^{N} w_{n}
            \mathcal{N}(Z|\mu_{n, \beta}, \Sigma_{n, \beta})
        \,.
    \end{align}
    where we define $\mu_{n, \beta} := R_{\beta} \mu_n$, and $\Sigma_{n, \beta} := R_{\beta} \Sigma_n R_{\beta}^{\top}$ for $n \in [N]$. 
    
    Then, using Lemma \ref{lemma:gauss_second_moment_ineq}, we obtain
    \begin{align}
        \tilde{G}_{\beta}
        & \succeq
            \sum_{n=1}^{N}
            w_n
            R_{\beta}^{\top}
            \left(
                \int
                Z Z^{\top}
                f_{\beta}(z_1)
                \mathcal{N}(X|\mu_{n, \beta}, \Sigma_{n, \beta})
                dZ
            \right)
            R_{\beta}
        \nonumber
        \\
        & \succeq
            \sum_{n=1}^{N}
            w_n
            c_n (\beta)
            R_{\beta}^{\top}
            \left(
                \Sigma_{n, \beta}
                +
                \mu_{n, \beta}
                \mu_{n, \beta}^{\top}
            \right)
            R_{\beta}
        \nonumber
        \\
        & =
            \sum_{n=1}^{N}
            w_n
            c_n (\beta)
            \left(
                \Sigma_{n}
                +
                \mu_{n} \mu_{n}^{\top}
            \right)
        \,.
    \end{align}
    Here $c_n(\beta) > 0$ is given by:
    \begin{align}
    c_n (\beta) :=
        \frac{1}{2}
        \left(
            2 g_{1n,\beta}
            + g_{3n,\beta}
            - \sqrt{
                g^{2}_{3n,\beta} + 4 g^{2}_{2n,\beta}}
        \right)
        \,,
        \label{eq:c_beta}
    \end{align}
    where
    \begin{align}
        g_{1n,\beta} & :=
            \int_{-\infty}^{\infty}
            \phi(z)
            f_{\beta}
            \left(
                \sqrt{(\Sigma_{n,\beta})_{11}} z
                + 
                (\mu_{n, \beta})_1
            \right)
            dz
        \,,
        \nonumber
        \\
        g_{2n,\beta} & :=
            \int_{-\infty}^{\infty}
            z
            \phi(z)
            f_{\beta}
            \left(
                \sqrt{(\Sigma_{n,\beta})_{11}} z
                + 
                (\mu_{n, \beta})_1
            \right)
            dz
        \,,
        \nonumber
        \\
        g_{3n,\beta} & :=
            \int_{-\infty}^{\infty}
            (z^2 - 1)
            \phi(z)
            f_{\beta}
            \left(
                \sqrt{(\Sigma_{n,\beta})_{11}} z
                + 
                (\mu_{n, \beta})_1
            \right)
            dz
        \,.
        \label{eq:g_n_beta}
    \end{align}
\end{proof}

\begin{lemma}
    \label{lemma:tilde_G}
    Under Assumptions
    \ref{assumption:arm_selection}
    and
    \ref{assumption:indep_arm},
    $\tilde{G}_{\beta}$
    is bounded as:
    \begin{align}
        \tilde{G}_{\beta}
        \succeq
            \int
            R_{\beta}^{\top}
            Z Z^{\top}
            R_{\beta}
            f_{\beta}(z_1)
            P_i(R_{\beta}^{\top} Z \mid \theta)
            dZ
        \,,
    \end{align}
    for any $\beta \in \mathbb{R}^{d}$. 
    Here 
    $Z = (z_1, \dots, z_d) \in \mathbb{R}^d$,
    $f_{\beta}: \mathbb{R} \rightarrow \mathbb{R}_{\geq 0}$ is a non-negative Lebesgue integrable function defined by Eq.~\eqref{eq:P_to_f},
    and
    $R_{\beta} \in \mathbb{R}^{d \times d}$ is an orthogonal matrix that satisfies $R_{\beta} \beta = (\| \beta \|_2, 0, \dots, 0)^{\top}$. 
\end{lemma}
\begin{proof}
    Using Assumptions \ref{assumption:arm_selection} and \ref{assumption:indep_arm}, we obtain
    \begin{align}
        \tilde{G}_{\beta}
        & :=
            \sum_{k} \int 
            X_{k} X_{k}^{\top} 
            P(
                \text{Select $k$}
                \mid
                \mathcal{X},
                \beta
            )
            P(\mathcal{X})
            \prod_{k'} dX_{k'}
        \nonumber \\
        & \succeq
            \int
            X_{i} X_{i}^{\top}
            \left(
                \int 
                P(
                    \text{Select $i$}
                    \mid
                    \mathcal{X},
                    \beta
                )
                P(\mathcal{X}\backslash \{X_{i}\})
                \prod_{k' \neq i} dX_{k'}
            \right)
            P_i(X_{i})
            dX_{i}
        \,,
    \end{align}
    The formula in the large brackets in the last line is 
    the marginalized PDF
    $P(\text{Select $i$} \mid X_i, \beta)$ and
    \begin{align}
        P(\text{Select $i$} \mid X_i, \beta)
        & =
            \int 
            P(
                \text{Select $i$}
                \mid
                \mathcal{X},
                \beta
            )
            P(\mathcal{X}\backslash \{X_{i}\})
            \prod_{k' \neq i} dX_{k'}
        \nonumber
        \\
        & =
            \int 
            \frac{
                \prod_{j \neq i}
                I\left[
                    X_i^{\top} \beta
                    \geq
                    X_j^{\top} \beta
                \right]
            }
            {
                \sum_{i'=1}^{K}
                \prod_{j \neq i'}
                I\left[
                    X_{i'}^{\top} \beta
                    \geq
                    X_j^{\top} \beta
                \right]
            }
            P(\mathcal{X}\backslash \{X_{i}\})
            \prod_{k' \neq i} dX_{k'}
        \nonumber
        \\
        & =
            \int 
            \frac{
                \prod_{j \neq i}
                I\left[
                    (X_i)_{\beta}
                    \geq
                    (X_j)_{\beta}
                \right]
            }
            {
                \sum_{i'=1}^{K}
                \prod_{j \neq i'}
                I\left[
                    (X_{i'})_{\beta}
                    \geq
                    (X_j)_{\beta}
                \right]
            }
            P(\mathcal{X}\backslash \{X_{i}\})
            \prod_{k' \neq i} dX_{k'}
    \end{align}
    where 
    $
        (X_i)_{\beta} := X_i^{\top} \beta / \|\beta\|_2
    $. 
    In the last line, we see that the dependence of $X_i$ only appears in $(X_i)_{\beta}$.
    Therefore, it can be written by a non-negative Lebesgue integrable function $f_{\beta}: \mathbb{R} \rightarrow \mathbb{R}$ as:
    \begin{align}
        \label{eq:P_to_f}
        f_{\beta} ((X_i)_{\beta})
        & :=
            P(
                \text{Select $i$}
                \mid
                X_i, \beta
            )
        \nonumber
        \\
        & =
            \int 
            \frac{
                \prod_{j \neq i}
                I\left[
                    (X_i)_{\beta}
                    \geq
                    (X_j)_{\beta}
                \right]
            }
            {
                \sum_{i'=1}^{K}
                \prod_{j \neq i'}
                I\left[
                    (X_{i'})_{\beta}
                    \geq
                    (X_j)_{\beta}
                \right]
            }
            P(\mathcal{X}\backslash \{X_{i}\})
            \prod_{k' \neq i} dX_{k'}
        \,.
    \end{align}
    Finally, transformation $Z := (z_1, \dots, z_d)^{\top} = R_{\beta} X_{i}$ gives:
    \begin{align}
        \tilde{G}_{\beta}
        \succeq
            \int
            R_{\beta}^{\top}
            Z Z^{\top}
            R_{\beta}
            f_{\beta}(z_1)
            P_i(R_{\beta}^{\top} Z)
            dZ
        \,,
    \end{align}
    which concludes the proof.
\end{proof}

\begin{lemma}
    \label{lemma:gauss_second_moment_ineq}
    For $Z = (z_1, \dots, z_d)^{\top} \in \mathbb{R}^{d}$, $\mu \in \mathbb{R}^{d}$, and positive definite matrix $\Sigma \in \mathbb{R}^{d \times d}$, 
    if a non-negative, real-valued, Lebesgue integrable function $f: \mathbb{R} \rightarrow \mathbb{R}_{\geq 0}$ is
    an element of $\mathcal{F}_1$ defined in Eq.~\eqref{eq:function_set}, 
    the following inequality holds:
    \begin{align}
        \int
        Z Z^{\top}
        \mathcal{N}(Z|\mu, \Sigma)
        f(z_1)
        dZ
        \succeq
            c
            \left(
                \Sigma
                +
                \mu \mu^{\top}
            \right)
        \,,
        \label{eq:disp_ineq}
    \end{align}
    Here $c$ is a positive constant defined as:
    \begin{align}
    c :=
        \frac{1}{2}
        \left(
            2 g_1
            + g_3
            - \sqrt{g^{2}_{3} + 4 g^{2}_{2}}
        \right)
        > 0
        \,,
        \label{eq:c_2}
    \end{align}
    where
    \begin{align}
        g_1 & :=
            \int_{-\infty}^{\infty}
            \phi(z)
            f\left(
                \sqrt{\Sigma_{11}} z
                + 
                \mu_1
            \right)
            dz
        \,,
        \nonumber
        \\
        g_2 & :=
            \int_{-\infty}^{\infty}
            z
            \phi(z)
            f\left(
                \sqrt{\Sigma_{11}} z
                + 
                \mu_1
            \right)
            dz
        \,,
        \nonumber
        \\
        g_3 & :=
            \int_{-\infty}^{\infty}
            (z^2 - 1)
            \phi(z)
            f\left(
                \sqrt{\Sigma_{11}} z
                + 
                \mu_1
            \right)
            dz
        \,.
        \label{eq:g_n}
    \end{align}
    Here we define the standard normal distribution $\phi(x):= \mathcal{N} (x \mid 0, 1)$.
\end{lemma}
\begin{proof}
    Because $f$ is Lebesgue integrable, using a parameter vector $A \in \mathbb{R}^{d}$, the integral can be computed as follows:
    \begin{align}
            \int
            z_i z_j
            \mathcal{N}(Z|\mu, \Sigma)
            f(z_1)
            dZ
        &
        =
            \lim_{|A| \rightarrow 0}
            \partial_{A_i} 
            \partial_{A_j} 
            \int
            e^{A^{\top} Z}
            \mathcal{N}(Z|\mu, \Sigma)
            f(z_1)
            dZ
        \nonumber
        \\
        &
        =
            \lim_{|A| \rightarrow 0}
            \partial_{A_i} 
            \partial_{A_j} 
            e^{1/2 A^{\top} \Sigma A + A^{\top} \mu}
            \int
            \mathcal{N}(Z|\mu + A^{\top} \Sigma, \Sigma)
            f(z_1)
            dZ
        \nonumber
        \\
        &
        =
            \lim_{|A| \rightarrow 0}
            \partial_{A_i} 
            \partial_{A_j} 
            e^{1/2 A^{\top} \Sigma A + A^{\top} \mu}
            \int_{-\infty}^{\infty}
            \mathcal{N}(z_1|(\mu + A^{\top} \Sigma)_1, \Sigma_{11})
            f(z_1)
            dz_1
        \,.
    \end{align}
    In the last line, we marginalize $z_2, \dots, z_d$ and use the property about the marginalization of the Gaussian distribution.
    The derivative and limit of $A$ give the following expression:
    \begin{align}
            \int_{-\infty}^{\infty}
            z_i z_j
            \mathcal{N}(Z|\mu, \Sigma) f(z_1)
            dZ
        =
        g_1
        (
            \Sigma_{ij}
            + 
            \mu_{i}
            \mu_{j}
        )
        +
        \mu_{i} f_j
        +
        \mu_{j} f_i
        +
        f_{ij}
        \,,
    \end{align}
    where
    \begin{align}
        g_1 & :=
            \lim_{|A| \rightarrow 0}
            \int_{-\infty}^{\infty}
            \mathcal{N}(z_1|(\mu + A^{\top} \Sigma)_1, \Sigma_{11})
            f(z_1)
            dz_1
        \nonumber
        \\
            & =
            \int_{-\infty}^{\infty}
            \phi(x)
            f\left(
                \sqrt{\Sigma_{11}} x
                + 
                \mu_1
            \right)
            dx
            \,,
    \end{align}
    \begin{align}
        f_i & := 
            \lim_{|A| \rightarrow 0}
            \partial_{A_i} 
            \int_{-\infty}^{\infty}
            \mathcal{N}(z_1|(\mu + A^{\top} \Sigma)_1, \Sigma_{11})
            f(z_1)
            dz_1
        \nonumber
        \\
            & = 
            \lim_{|A| \rightarrow 0}
            \int_{-\infty}^{\infty}
            \frac{
                \Sigma_{i1} (z_1 - (\mu + A^{\top} \Sigma)_1)
            }{\Sigma_{11}}
            \mathcal{N}(z_1|(\mu + A^{\top} \Sigma)_1, \Sigma_{11})
            f(z_1)
            dz_1
        \nonumber
        \\
            & = 
            \frac{\Sigma_{i1}}{\sqrt{\Sigma_{11}}}
            \int_{-\infty}^{\infty}
            \frac{z_1 - \mu_1}{\sqrt{\Sigma_{11}}}
            \mathcal{N}(z_1|\mu_1, \Sigma_{11})
            f(z_1)
            dz_1
        \nonumber
        \\
            & = 
            \frac{\Sigma_{i1}}{\sqrt{\Sigma_{11}}}
            \int_{-\infty}^{\infty}
            x
            \phi(x)
            f (
                \sqrt{\Sigma_{11}} x
                + 
                \mu_1
            )
            dx
            \,,
    \end{align}
    and
    \begin{align}
        f_{ij} 
        & := 
            \lim_{|A| \rightarrow 0}
            \partial_{A_i} 
            \partial_{A_j} 
            \int_{-\infty}^{\infty}
            \mathcal{N}(z_1|(\mu + A^{\top} \Sigma)_1, \Sigma_{11})
            f(z_1)
            dz_1
        \nonumber
        \\
        & = 
            \lim_{|A| \rightarrow 0}
            \partial_{A_i} 
            \int_{-\infty}^{\infty}
            \frac{
                \Sigma_{j1} (z_1 - (\mu + A^{\top} \Sigma)_1)
            }{\Sigma_{11}}
            \mathcal{N}(z_1|(\mu + A^{\top} \Sigma)_1, \Sigma_{11})
            f(z_1)
            dz_1
        \nonumber
        \\
        & = 
            \lim_{|A| \rightarrow 0}
            \int_{-\infty}^{\infty}
            \left(
                -
                \frac{
                    \Sigma_{i1} 
                    \Sigma_{j1} 
                }{\Sigma_{11}}
                +
                \frac{
                    \Sigma_{i1} 
                    \Sigma_{j1} 
                    (z_1 - (\mu + A^{\top} \Sigma)_1)^2
                }{\Sigma_{11}^2}
            \right)
            \mathcal{N}(z_1|(\mu + A^{\top} \Sigma)_1, \Sigma_{11})
            f(z_1)
            dz_1
        \nonumber
        \\
        & = 
            \int_{-\infty}^{\infty}
            \left(
                -
                \frac{
                    \Sigma_{i1} 
                    \Sigma_{j1} 
                }{\Sigma_{11}}
                +
                \frac{
                    \Sigma_{i1} 
                    \Sigma_{j1} 
                    (z_1 - \mu_1)^2
                }{\Sigma_{11}^2}
            \right)
            \mathcal{N}(z_1|\mu_1, \Sigma_{11})
            f(z_1)
            dz_1
        \nonumber
        \\
        & = 
            \frac{
                \Sigma_{i1}
                \Sigma_{j1}
            }
            {\Sigma_{11}}
            \int_{-\infty}^{\infty}
            (-1 + x^2)
            \phi(x)
            f\left(
                \sqrt{\Sigma_{11}} x
                + 
                \mu_1
            \right)
            dx
        \,.
    \end{align}
    In the last equality of each equation, we convert 
    $ x = 
        \left(
            z_1 
            - \mu_1
        \right) / \sqrt{\Sigma_{11}}
    $.
    Using $g_1$, $g_2$, and $g_3$ given by Eq~\eqref{eq:g_n}, for $V \in \mathbb{R}^d$ we obtain:
    \footnote{
        We here note that in the limit $\sqrt{(\Sigma)_{11}} \rightarrow 0$, $\Sigma_{1\cdot} / \sqrt{\Sigma_{11}}$ converges to a finite value because $\Sigma_{1\cdot}$ also approaches $0$ due to the positive (semi-)definiteness of $\Sigma$, as shown in Lemma \ref{lemma:zero_limit_positive_semidefinite}. 
    }
    \begin{align}
        V^{\top} 
        \int
        Z Z^{\top}
        \mathcal{N}(Z|\mu, \Sigma)
        f(z_1)
        dZ
        V
        & =
            V^{\top}
            \left(
                g_1 \Sigma
                + 
                g_1 \mu \mu^{\top}
                +
                g_2
                \frac{
                    \Sigma_{\cdot 1} \mu^{\top}
                    +
                    \mu \Sigma_{\cdot 1}^{\top} 
                }
                {
                    \sqrt{\Sigma_{11}}
                }
                +
                g_3
                \frac{
                    \Sigma_{\cdot 1}
                    \Sigma_{\cdot 1}^{\top}
                }{
                    \Sigma_{11}
                }
            \right)
            V
        \nonumber
        \\
        & \hspace{-50pt} =
            g_1 V^{\top} \Sigma V
            +
            g_1 (V^{\top} \mu)^2
            +
            2 g_2
            (V^{\top} \mu)
            (
                r^{\top}
                \Sigma^{1/2}
                V
            )
            + 
            g_3
            \left(
                r^{\top}
                \Sigma^{1/2}
                V
            \right)^2
        \,.
        \label{eq:lb2}
    \end{align}
    In the last line, we define
    $
    r := 
        \Sigma^{1/2}_{\cdot 1} /
        \sqrt{\Sigma_{11}}
    $.
    By the definition, $r$ satisfies the condition $\| r \|_2 = 1$ and therefore we can write
    $
        r^{\top}
        \Sigma^{1/2}
        V
        :=
        \sqrt{
        V^{\top}
        \Sigma
        V
        }
        \cos \theta
    $.
    Then, Eq.~\eqref{eq:lb2} can be simplified as the following quadratic form:
    \begin{align}
        (
            g_1 
            +
            g_3 \cos^2 \theta
        )
        x^2
        +
        g_1 y^2
        +
        2 g_2 \cos \theta
        x y
        \label{eq:quadratic}
    \,,
    \end{align}
    where we replace $\sqrt{V^{\top} \Sigma V} = x$ and $V^{\top} \mu = y$. Standardization of this quadratic form derives the following expression:
    \begin{align}
        &
        \frac{
            2 g_1
            + g_3 \cos^2 \theta
            + D
        }{4D}
        \left(
            \sqrt{D + g_3 \cos^2 \theta} x
            +
            \sqrt{D - g_3 \cos^2 \theta} y
        \right)^2
        \nonumber
        \\
            & \hspace{50pt} +
            \frac{
                2 g_1
                + g_3 \cos^2 \theta
                - D
            }{4D}
            \left(
                \sqrt{D - g_3 \cos^2 \theta} x
                -
                \sqrt{D + g_3 \cos^2 \theta} y
            \right)^2
        \nonumber
        \\
        \geq &
        \frac{
            2 g_1
            + g_3 \cos^2 \theta
            - D
        }{4D}
        \left(
            \left(
                \sqrt{D + g_3 \cos^2 \theta} x
                +
                \sqrt{D - g_3 \cos^2 \theta} y
            \right)^2
            \right.
            \nonumber
            \\
            & \hspace{150pt}
            +
            \left.
                \left(
                    \sqrt{D - g_3 \cos^2 \theta} x
                    -
                    \sqrt{D + g_3 \cos^2 \theta} y
                \right)^2
            \right)
        \nonumber
        \\
        = &
        \frac{
            2 g_1
            + g_3 \cos^2 \theta
            - D
        }{2}
        (x^2 + y^2)
        \nonumber
        \\
        \geq &
        \frac{
            2 g_1
            + g_3
            - \sqrt{g^{2}_{3} + 4 g^{2}_{2}}
        }{2}
        (x^2 + y^2)
        := c (x^2 + y^2)
        \nonumber
    \,,
    \end{align}
    where we define $D := \sqrt{
        g^{2}_{3} \cos^4 \theta
        +
        4 g^{2}_{2} \cos^2 \theta
    }$ for brevity and we take the minimum value for $\theta$ in the last line.
    Finally, Lemma \ref{lemma:c_positive} shows $c > 0$, which finishes the proof.
\end{proof}

\begin{lemma}
    \label{lemma:c_positive}
    If a non-negative, real-valued, Lebesgue integrable function $f: \mathbb{R} \rightarrow \mathbb{R}_{\geq 0}$ 
    is an element of $\mathcal{F}_1$ defined in Eq.~\eqref{eq:function_set}, 
    and if $\Sigma_{11} > 0$, 
    then the coefficient $c$ given by Eq.~\eqref{eq:c_2} is positive.
\end{lemma}
\begin{proof}
    Let us write 
    $
        g(z) := f(
            \sqrt{\Sigma_{11}} z + \mu_1
        )
    $ 
    for brevity and define 
    $
        \mathcal{R} := \{z \in \mathbb{R} \mid f(z) \ge \epsilon\}
    $, where $\epsilon$ is given in Eq.~\eqref{eq:function_set}.
    By definition, $\mathcal{R}$ has a non-zero measure.
    We also define
    $
        \mathcal{R}' := 
        \{
            \sqrt{\Sigma_{11}} z + \mu_1 \in \mathcal{R} 
            \mid
            z \in \mathbb{R}
        \}
    $.
    Because $\Sigma_{11} > 0$, 
    $\mathcal{R}'$ also has a non-zero measure and therefore $g \in \mathcal{F}_1$. Then, it is immediate that $g_1$ is positive. 
    Moreover,
    $2g_1 + g_3$ is positive because
    \begin{align}
        g_1 + g_3 
        = 
            \int z^2 \phi(z) g(z) dz
        \geq 
            \int_{\mathcal{R}'} z^2 \phi(z) g(z) dz
        > 0
        \,.
        \nonumber
    \end{align}
    The following integral is also positive:
    \begin{align}
        0
        & < 
            \int_{\mathcal{R}'}
            \left(
                z - \frac{g_2}{g_1}
            \right)^2
            \phi(z) g(z) 
            dz
            \leq
                \int
                \left(
                    z - \frac{g_2}{g_1}
                \right)^2
                \phi(z) g(z) 
                dz
        \nonumber
        \\
        & = 
            \int
                z^2 \phi(z) g(z) 
                dz
            -
            \frac{2 g_2}{g_1}
            \int
                z \phi(z) g(z) 
                dz
            +
            \frac{g_2^2}{g_1^2}
            \int
                \phi(z) g(z) 
                dz
        \nonumber
        \\
        & =
            g_1 + g_3
            -
            \frac{g_2^2}{g_1}
            \,,
    \end{align}
    which gives $(g_1 + g_3)g_1 > g_2^2$. These inequalities conclude $c > 0$ as follows:
    \begin{align}
        2 c = 
        2 g_1 + g_3 - 
        \sqrt{g_3^2 + 4 g_2^2}
        & >
            2 g_1 + g_3 -
            \sqrt{
                g_3^2
                +
                4 g_1^2
                + 
                4 g_1 g_3
            }
        \nonumber
        \\
        & =
            2 g_1 + g_3 -
            \sqrt{
                (2 g_1 + g_3)^2
            }
        \nonumber
        \\
        & = 0 
        \,.
        \nonumber
    \end{align}
\end{proof}

\begin{lemma}
    \label{lemma:zero_limit_positive_semidefinite}
    For a positive semi-definite matrix $\Sigma \in \mathbb{R}^{d \times d}$,
    the vectors
    $(\Sigma)_{\cdot d} / (\Sigma)_{{dd}}$
    and
    $(\Sigma)_{d \cdot} / (\Sigma)_{{dd}}$
    converge to a finite value
    in the limit $(\Sigma)_{dd} \rightarrow 0$.
\end{lemma}
\begin{proof}
    For the positive semi-definite matrix $\Sigma$, there exists vectors $v_i \in \mathbb{R}^{d}$, $i \in [d]$ that satisfy $(\Sigma)_{ij} = v_i^{\top} v_j$. Therefore, $\sqrt{(\Sigma)_{dd}} \rightarrow 0$ implies $\|v_d\|_2 \rightarrow 0$, in which $(\Sigma)_{id} / \sqrt{(\Sigma)_{dd}}$ converges to  $\|v_i\|_2 \cos \theta_{id}$, where $\theta_{id}$ is the angle between $v_i$ and $v_d$.
\end{proof}

\subsection{Proof of Lemma \ref{lemma:gmm_coef}}
\begin{proof}
    If $\beta^T \Sigma_n \beta > 0$, then $(\Sigma_{n, \beta})_{11} > 0$.
    Therefore, from Lemma \ref{lemma:c_positive}, $c_n(\beta) > 0$.
    If $(\Sigma_{n, \beta})_{11} \rightarrow 0$,
    then 
    from Eq.~\eqref{eq:g_n_beta}, 
    the coefficients are given by
    \begin{align}
        g_{1n,\beta} & \rightarrow 
            f_{\beta}
            \left(
                (\mu_{n, \beta})_1
            \right)
        \,,
        \nonumber
        \\
        g_{2n,\beta} & \rightarrow
            f_{\beta}
            \left(
                (\mu_{n, \beta})_1
            \right)
            \int x \phi(x) dx
            =
            0
        \,,
        \nonumber
        \\
        g_{3n,\beta} & \rightarrow
            f_{\beta}
            \left(
                (\mu_{n, \beta})_1
            \right)
            \int (x^2 -1) \phi(x) dx
            =
            0
        \,.
        \label{eq:g_n_limit}
    \end{align}
    Therefore,
    $ c_n (\beta) = 
        f_{\beta}\left(
            (\mu_{n, \beta})_1
        \right)
    $.
    Finally, from Eq.~\eqref{eq:P_to_f},
    \begin{align}
        f_{\beta} (
            (\mu_{n, \beta})_1
        )
        =
        f_{\beta} (
            (\mu_{n})_{\beta}
        )
        =
        P(
            \text{Select $i$}
            \mid
            X_i = \mu_{n}, \beta
        )
        \,,
    \end{align}
    which concludes the proof.
\end{proof}

\subsection{Proof of Theorem \ref{theorem:low_rank_gaussian_mixture}}
\begin{proof}
    \label{proof:low_rank_gaussian_mixture}
    Similar to the proof of Lemma \ref{lemma:gm_lower_bound}, 
    from Lemma \ref{lemma:Gt}, 
    it is sufficient to show that there exists a constant $\phi_0 > 0$ that satisfies $\phi_{\mathcal{S}}(\tilde{G}_{\beta}) > \phi_0$ for any $\beta \in \mathbb{R}^d$.
    From Lemma \ref{lemma:gm_lower_bound},
    we have
    \begin{align}
        \phi_{\mathcal{S}}
        \left(
            \tilde{G}_{\beta}
        \right)
        & \geq
            \phi_{\mathcal{S}}
            \left(
                \sum_{n} 
                c_n(\beta)
                w_n (\Sigma_{n}^{d_n} + \mu_{n} \mu_{n}^{\top}) 
            \right)
        \nonumber
        \\
        & \geq
            \inf\limits_{\beta \in \mathbb{R}^d}
            \phi_{\mathcal{S}}
            \left(
                \sum_{n} 
                c_n(\beta)
                w_n (\Sigma_{n}^{d_n} + \mu_{n} \mu_{n}^{\top}) 
            \right)
        \geq
        \phi_0
    \,.
    \end{align}
\end{proof}

\subsection{Proof of Theorem \ref{theorem:radial}}
\label{proof:radial}
\begin{proof}
    Similar to the proof of Lemma \ref{lemma:gm_lower_bound}, from Lemma \ref{lemma:Gt}, it is sufficient to show that there exists constant $\phi_0 > 0$ that satisfies $\phi_{\mathcal{S}}(\tilde{G}_{\beta}) > \phi_0$ for any $\beta \in \mathbb{R}^d$.
    Under Assumptions 
    \ref{assumption:arm_selection}
    and
    \ref{assumption:indep_arm}, 
    Lemma \ref{lemma:tilde_G} gives:
    \begin{align}
        \tilde{G}_{\beta}
        \succeq
            \int
            R_{\beta}^{\top}
            Z Z^{\top}
            R_{\beta}
            f_{\beta}(z_1)
            P_{D}(R_{\beta}^{\top} Z)
            dZ
        \,,
        \label{eq:tilde_G_lower_D}
    \end{align}
    where 
    $Z = (z_1, \dots, z_d) \in \mathbb{R}^d$,
    $f_{\beta}: \mathbb{R} \rightarrow \mathbb{R}_{\geq 0}$ is a non-negative Lebesgue integrable function defined by Eq.~\eqref{eq:P_to_f},
    and
    $R_{\beta} \in \mathbb{R}^{d \times d}$ is an orthogonal matrix that satisfies $R_{\beta} \beta = (\| \beta \|_2, 0, \dots, 0)^{\top}$. 
    Here, $P_{D}(R_{\beta}^{\top} Z)$ can be written as:
    \begin{align}
        P_{D} (R_{\beta} Z) 
        = 
            \sum_{n=1}^{N} w_{n}
            Q_n (\|
                R_{\beta}^{\top} Z - \mu_{n}
            \|_2)
        = 
            \sum_{n=1}^{N} w_{n}
            Q_n (\|
                Z - \mu_{n, \beta}
            \|_2)
        \,,
    \end{align}
    where we define $\mu_{n, \beta} := R_{\beta} \mu_n$ for $n \in [N]$. 

    From the discussion around Eq.~\eqref{eq:selection_possibility} in the proof of Lemma \ref{lemma:gm_lower_bound},
    Assumptions 
    \ref{assumption:arm_selection} 
    and
    \ref{assumption:selection_possibility} ensure that
    the function
    $
        Z \rightarrow
        f_{\beta} (z_1)
        P_{D}(R_{\beta}^{\top} Z)
    $ 
    is in $\mathcal{F}_d$ defined in Eq.~\eqref{eq:function_set}.\footnote{
        In the proof, we only consider distributions that have density functions. We note that such distributions can approximate distributions that do not have a density function with arbitrary precision.
    }
    Because $P_{D}$ is a mixture distribution of $Q_n$s, it implies that there exists a set $\mathcal{M} \subset [N]$ such that 
    the functions
    $
        Z \rightarrow
        f_{\beta} (z_1)
        Q_n(
            \|Z - \mu_{n, \beta}\|_2
        )
    $ are in $\mathcal{F}_d$ for $n \in \mathcal{M}$.
    
    Then,
    \begin{align}
        \phi_{\mathcal{S}} 
        \left(
            \tilde{G}_{\beta} (P)
        \right)
        & \geq
            \phi_{\mathcal{S}} 
            \left(
                \sum_{n=1}^{N}
                w_n
                R_{\beta}^{\top}
                \left(
                    \int
                    Z Z^{\top}
                    f_{\beta} (z_1)
                    Q_n (\|
                        Z - \mu_{n, \beta}
                    \|_2)
                    dZ
                \right)
                R_{\beta}
            \right)
        \nonumber
        \\
        &  \geq
            \sum_{n \in \mathcal{M}}
            w_n
            \phi_{\mathcal{S}} 
            \left(
                R_{\beta}^{\top}
                \int
                Z Z^{\top}
                f_{\beta} (z_1)
                Q_n (\|
                    Z - \mu_{n, \beta}
                \|_2)
                dZ
                R_{\beta}
            \right)
        \,.
    \end{align}
    In the last line, we use Lemma \ref{lemma:cc_ineq}.
    From Lemma \ref{lemma:radial_second_moment_ineq}, 
    the matrix in the bracket in the last line is a positive definite matrix, which concludes the proof.
\end{proof}

\begin{lemma}
    \label{lemma:radial_second_moment_ineq}
    If a non-negative Lebesgue integrable function 
    $
        f :\mathbb{R} \rightarrow \mathbb{R}_{\geq 0}
    $
    is an element of $\mathcal{F}_1$,
    and if the function 
    $ X \rightarrow f(x_1) Q(\|X - \mu\|_2) $ 
    is an element of $\mathcal{F}_d$,
    then
    \begin{align}
        \int X X^{\top} 
        f_{\beta}(x_1) 
        Q(\|X - \mu\|_2) 
        dX
        \succ 0
    \,.
    \end{align}
\end{lemma}
\begin{proof}
    Converting $Z = X - \mu$, we obtain
    \begin{align}
        \int X X^{\top} f(x_1) 
        Q(\|X - \mu\|_2) dX
        = 
            \int
                (Z + \mu) (Z + \mu)^{\top} 
                f(z_1 + \mu_1)
                Q(\|Z\|_2) 
            dZ
        \,.
    \end{align}
    We define the following symbols:
    \begin{align}
        \tilde{c}
            & := 
                \int
                f(z_1 + \mu_1) 
                Q(\|Z\|_2)
                dZ
        \,,
        \nonumber
        \\
        \tilde{\mu}_1
            & := 
                \frac{1}{\tilde{c}}
                \int
                z_1
                f(z_1 + \mu_1) 
                Q(\|Z\|_2)
                dZ
        \,,
        \nonumber
        \\
        \tilde{\sigma}_1^2
            & := 
                \frac{1}{\tilde{c}}
                \int
                (z_1 - \tilde{\mu}_1)^2
                f(z_1 + \mu_1) 
                Q(\|Z\|_2)
                dZ
            = 
                \frac{1}{\tilde{c}}
                \int
                z_1^2
                f(z_1 + \mu_1) 
                Q(\|Z\|_2)
                dZ
                -
                \tilde{\mu}_1^2
        \,,
        \nonumber
        \\
        \tilde{\sigma}_i^2
            & := 
                \frac{1}{\tilde{c}}
                \int
                z_i^2
                f(z_1 + \mu_1) 
                Q(\|Z\|_2)
                dz
            \hspace{10pt}
            \text{(for $i = 2, \dots , d$)}
        \,.
    \end{align}
    Here, $\tilde{c}, \tilde{\sigma}_1^2, \tilde{\sigma}_i^2 > 0$ because 
    the function
    $X \rightarrow f(x_1) Q(\|X - \mu\|_2) $ is in $\mathcal{F}_d$.
    We note that $\tilde{\mu}_i := \int z_i f(z_1) Q(\|Z\|_2) dz = 0$ because $z_i Q(\|Z\|_2)$ is an odd function for $z_i$.
    For indices $1, i, j$, $2 \leq i, j \leq d$, the integrals are computed as:
    \begin{align}
        \int 
        (z_1 + \mu_1)^2 f(z_1 + \mu_1) Q(\|Z\|_2) dZ
        & = 
            \int 
            (z_1^2 + 2 \mu_1 z_1 + \mu_1^2)
            f(z_1 + \mu_1) 
            Q(\|Z\|_2)
            dZ
        \nonumber
        \\
        & = 
            \tilde{c}
            \left(
                \tilde{\sigma}_1^2 + \tilde{\mu}_1^2
                +
                2 \mu_1 \tilde{\mu}_1
                +
                \mu_1^2
            \right)
        \nonumber
        \\
        & = 
            \tilde{c}
            \left(
                \tilde{\sigma}_1^2 
                +
                (\tilde{\mu}_1 + \mu_1)^2
            \right)
        \,,
        \nonumber
        \\
        \int 
        (z_1 + \mu_1) (z_j + \mu_j) f(z_1 + \mu_1) 
        Q(\|Z\|_2) dZ
        & =
            \int 
            (z_1 z_j + \mu_1 z_j + z_1 \mu_j + \mu_1 \mu_j) f(z_1 + \mu_1) 
            Q(\|Z\|_2) dZ
        \nonumber
        \\
        & =
            \int 
            (z_1 \mu_j + \mu_1 \mu_j) f(z_1 + \mu_1) 
            Q(\|Z\|_2) dZ
        \nonumber
        \\
        & =
            \tilde{c}
            \left(
                \tilde{\mu}_1
                +
                \mu_1
            \right)
            \mu_j
        \,,
        \nonumber
        \\
        \int (z_i + \mu_i) (z_j + \mu_j) f(z_1 + \mu_1) Q(\|Z\|_2) dZ
        & = 
            \int (z_i^2 \delta_{ij} + \mu_i \mu_j) f(z_1 + \mu_1) Q(\|Z\|_2) dZ
        \nonumber
        \\
        & =
            \tilde{c}
            \left(
                \tilde{\sigma}_i^2 \delta_{ij}
                +
                \mu_i \mu_j
            \right)
        \,,
    \end{align}
    respectively.
    In summary, we obtain
    \begin{align}
        \int X X^{\top} f(x_1) 
        Q(\|X - \mu\|_2) dX
        =
        \tilde{c}
        \left(
            \mathrm{diag}(\tilde{\sigma}_1^2, \dots, \tilde{\sigma}_d^2)
            +
            \tilde{\mu} \tilde{\mu}^{\top}
        \right)
        \,,
    \end{align}
    where 
    $
        \tilde{\mu} :=
        \{
            \mu_1 + \tilde{\mu}_1, \mu_2, \dots, \mu_d
        \}
    $.
    Since $\tilde{c}, \tilde{\sigma}_1^2, \tilde{\sigma}_i^2 > 0$, the first term proves the claim.
\end{proof}

\subsection{Proof of Lemma \ref{lemma:cc_sum}}
\begin{proof}
    \label{proof:cc_sum}
    For any $V \in \mathbb{R}^{d}$, 
    the positive semi-definiteness of the matrices gives $V^{\top} \Lambda_n V \geq 0$.
    If any $V \in \mathcal{D}$ satisfies $V^{\top} (\sum_n w_n \Lambda_n) V > 0$, there exists $m \in [N]$ such that $V^{\top} \Lambda_{m} V > 0$ for each $V$,
    and therefore $ V^{\top} (\sum_n w'_{n} \Lambda_{n}) V \geq V^{\top} w'_{m} \Lambda_{m} V > 0$ for any $V \in \mathcal{D}$.
\end{proof}

\section{Summary of the lower-bound formulae}
We summarize the formulae for the lower bound of each basis introduced in the previous section.
First, in Eq.~\eqref{eq:P_to_f}, a non-negative Lebesgue integrable function $f_{\beta}: \mathbb{R} \rightarrow \mathbb{R}$ is defined as:
\begin{align}
    f_{\beta} ((X_i)_{\beta})
    & :=
        P(
            \text{Select $i$}
            \mid
            X_i, \beta
        )
    \nonumber
    \\
    & =
        \int 
        \frac{
            \prod_{j \neq i}
            I\left[
                (X_i)_{\beta}
                \geq
                (X_j)_{\beta}
            \right]
        }
        {
            \sum_{i'=1}^{K}
            \prod_{j \neq i'}
            I\left[
                (X_{i'})_{\beta}
                \geq
                (X_j)_{\beta}
            \right]
        }
        P(\mathcal{X}\backslash \{X_{i}\})
        \prod_{k' \neq i} dX_{k'}
    \,.
\end{align}
\begin{remark}
    From
    Lemma \ref{lemma:gm_lower_bound}, 
    Eq.~\eqref{eq:c_beta},
    Eq.~\eqref{eq:g_n_beta},
    and
    Eq.~\eqref{eq:g_n_limit},
    under Assumptions 
    \ref{assumption:arm_selection},
    \ref{assumption:indep_arm},
    and
    \ref{assumption:selection_possibility},
    if $P_i(X)$ is the 
    Gaussian mixture / low-rank Gaussian mixture / discrete basis 
    then the following lower bound holds:
    \begin{align}
        \phi_{\mathcal{S}} 
        \left(
            \tilde{G}_{\beta}
        \right)
        \geq
        \phi_{\mathcal{S}} 
        \left(
            \sum_{n=1}^{N}
            w_n 
            c_n(\beta)
            (\Sigma_n + \mu_n \mu_n^{\top})
        \right)
        \,,
    \end{align}
    where $\Sigma_n$ may be low-rank for the low-rank Gaussian mixture basis and is zero for the discrete basis.  The coefficient is given by:
    \begin{align}
    c_n (\beta) :=
        \frac{1}{2}
        \left(
            2 g_{1n,\beta}
            + g_{3n,\beta}
            - \sqrt{
                g^{2}_{3n,\beta} + 4 g^{2}_{2n,\beta}}
        \right)
        \,,
    \end{align}
    where
    \begin{align}
        g_{1n,\beta} & :=
            \int_{-\infty}^{\infty}
            \phi(z)
            f_{\beta}
            \left(
                \sqrt{(\Sigma_{n,\beta})_{11}} z
                + 
                (\mu_{n, \beta})_1
            \right)
            dz
        \,,
        \nonumber
        \\
        g_{2n,\beta} & :=
            \int_{-\infty}^{\infty}
            z
            \phi(z)
            f_{\beta}
            \left(
                \sqrt{(\Sigma_{n,\beta})_{11}} z
                + 
                (\mu_{n, \beta})_1
            \right)
            dz
        \,,
        \nonumber
        \\
        g_{3n,\beta} & :=
            \int_{-\infty}^{\infty}
            (z^2 - 1)
            \phi(z)
            f_{\beta}
            \left(
                \sqrt{(\Sigma_{n,\beta})_{11}} z
                + 
                (\mu_{n, \beta})_1
            \right)
            dz
        \,,
    \end{align}
    for $(\Sigma_{n,\beta})_{11} > 0$, and 
    \begin{align}
        g_{1n,\beta} 
        =
            f_{\beta}
            \left(
                (\mu_{n, \beta})_1
            \right)
        \,,
        \hspace{5pt}
        g_{2n,\beta} 
        & =
        g_{3n,\beta} 
        = 0
        \,,
    \end{align}
    for $(\Sigma_{n,\beta})_{11} \rightarrow 0$.
\end{remark}
\begin{remark}
    If $P_i(X)$ is the radial basis, the following lower bound holds:
    \begin{align}
        \phi_{\mathcal{S}} 
        \left(
            \tilde{G}_{\beta}
        \right)
        &  \geq
            \phi_{\mathcal{S}} 
            \left(
                \sum_{n=1}^{N}
                w_n
                \tilde{c}_{n, \beta}
                R_{\beta}^{\top}
            \left(
                \mathrm{diag}(
                    \tilde{\sigma}_{1n, \beta}^2, 
                    \dots, 
                    \tilde{\sigma}_{dn, \beta}^2
                )
                +
                \tilde{\mu}_{n, \beta} \tilde{\mu}_{n, \beta}^{\top}
            \right)
                R_{\beta}
        \right)
        \,,
    \end{align}
    where
    \begin{align}
        \tilde{c}_{n, \beta}
            & := 
                \int
                f_{\beta}(z_1 + (\mu_{n, \beta})_1) 
                Q(\|Z\|_2)
                dZ
        \,,
        \nonumber
        \\
        \tilde{\mu}_{1n, \beta}
            & := 
                \frac{1}{\tilde{c}}
                \int
                z_1
                f_{\beta}(z_1 + (\mu_{n, \beta})_1) 
                Q(\|Z\|_2)
                dZ
        \,,
        \nonumber
        \\
        \tilde{\sigma}_{1n, \beta}^2
            & := 
                \frac{1}{\tilde{c}}
                \int
                (z_1 - \tilde{\mu}_{1n, \beta})^2
                f_{\beta}(z_1 + (\mu_{n, \beta})_1) 
                Q(\|Z\|_2)
                dZ
            = 
                \frac{1}{\tilde{c}}
                \int
                z_1^2
                f_{\beta}(z_1 + (\mu_{n, \beta})_1) 
                Q(\|Z\|_2)
                dZ
                -
                \tilde{\mu}_{1n, \beta}^2
        \,,
        \nonumber
        \\
        \tilde{\sigma}_{in,\beta}^2
            & := 
                \frac{1}{\tilde{c}}
                \int
                z_i^2
                f_{\beta}(z_1 + (\mu_{n, \beta})_1) 
                Q(\|Z\|_2)
                dz
            \hspace{10pt}
            \text{(for $i = 2, \dots , d$)}
        \nonumber
        \\
        \tilde{\mu}_{n,\beta}
            & :=
            \{
                (\mu_{n, \beta})_1 + \tilde{\mu}_{1n,\beta}, 
                (\mu_{n, \beta})_2
                \dots, 
                (\mu_{n, \beta})_d
            \}
        \,.
    \end{align}
\end{remark}

\section{Numerical experiment}
\begin{figure}[t]
    \centering
    \includegraphics[]{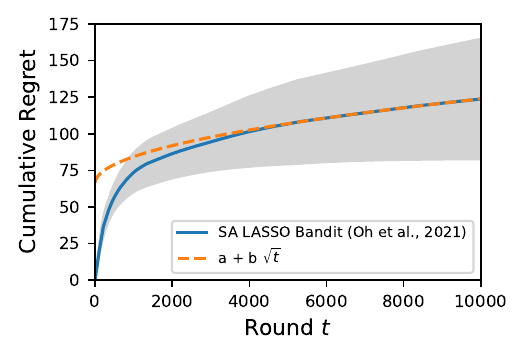}
    \caption{
        Cumulative regret of SA LASSO Bandit algorithm~\citep{DBLP:conf/icml/OhIZ21} on artificial data that do not satisfy the RS condition.
        The blue line represents the average of 100 trials, while the orange line shows the function $a + b \sqrt{t}$ fitted with the results for rounds $t=5000$ to $10000$.
        The shaded area represents the 0.5-$\sigma$ standard deviation region.
    }
    \label{fig:result}
\end{figure}
To empirically validate our claim, we performed a numerical experiment with artificial data. In this experiment, we set the number of arms $K = 3$, the arm feature dimension $d = 10$, and each arm feature was assumed to be generated from a uniform distribution of 0 to 1. In this situation, the support of the arm feature distribution is origin asymmetric and does not satisfy Assumption \ref{assumption:RS}. The reward was given by the inner product of the selected arm feature and the sparse parameter $\beta$ plus noise. The sparsity of parameter $\beta$ was set to 2, and the non-zero components were generated by a uniform distribution of 0 to 1. Gaussian noise with a variance of 0.1 was added to the reward.  We employed SA LASSO Bandit~\citep{DBLP:conf/icml/OhIZ21} as the bandit algorithm. The experiment was conducted 100 times with different $\beta$s. We show the result in Fig \ref{fig:result}. In the figure, the blue line shows the average value of the cumulative regrets, while the orange line shows the function $a + b \sqrt{t}$ fitted with the results for rounds $t=5000$ to $10000$. We observe that the greedy algorithm yields cumulative regret that asymptotically behaves like $\sqrt{T}$ even for the arm feature distribution that does not satisfy Assumption \ref{assumption:RS}.

\section{Proofs of the results in Section \ref{section:application}}

\subsection{Proof of Corollary \ref{corollary:combinatorial_bound_empirical}}
\begin{proof}
    \label{proof:combinatorial_bound_empirical}
    Similar to the proof of Lemma \ref{lemma:greedy_bound_empirical}, we define the following high probability event at round $t$:
    \begin{align}
        \mathcal{E}_{\mathrm{gap}}^t
        = \{
            \|\beta^* - \hat{\beta}_t \|_1 \leq 
            8\lambda_{Lt} s_0 / \bar{\phi}^2
        \}
        \,,
    \end{align}
    and below, we set $\lambda_{Lt} = 4 x_{\mathrm{max}} \sigma \sqrt{ (\delta^2 + 2 \log d) / L t}$.
    By Lemma \ref{lemma:lasso_estimator}, the above inequality is valid when both events Eq.~\eqref{eq:epsX} and Eq.~\eqref{eq:cc_btwn_empirical_expected} occur simultaneously.
    From Lemmas \ref{lemma:epsX} and \ref{lemma:cc_btwn_empirical_expected}, the probability of the event $\mathcal{E}_{\mathrm{gap}}^t$ is given by:
    \begin{align}
        P(\mathcal{E}_{\mathrm{gap}}^t)
        \geq
        1 - 2 e^{-\delta^2/2} - e^{-Lt \kappa(\bar{\phi})^2}\,,
    \end{align}
    for
    $ t \geq T_0 := \log(d (d+1))/L \kappa(\bar{\phi})$,
    where 
    $\kappa(\bar{\phi}) := \min (2 - \sqrt{2}, \bar{\phi}^2 / (256 x_{\mathrm{max}}^2 s_0))$.

    For the combinatorial setting, the reward $\mathrm{reg}(t)$ is bounded as follows:
    For each arm in the set $\mathcal{I}_t$ and $\mathcal{I}_t^*$ , we index it according to the size of its inner product with $\hat{\beta}_{t-1}$ as follows.
    \begin{align}
        X_{a_t^{1}} \hat{\beta}_{t-1}
        \geq
        X_{a_t^{2}} \hat{\beta}_{t-1}
        \geq
        \dots
        \geq
        X_{a_t^{L}} \hat{\beta}_{t-1}
    \,,
    \nonumber
    \\
        X_{a_t^{*1}} \hat{\beta}_{t-1}
        \geq
        X_{a_t^{*2}} \hat{\beta}_{t-1}
        \geq
        \dots
        \geq
        X_{a_t^{*L}} \hat{\beta}_{t-1}
    \,.
    \end{align}
    Then, due to the greedy selection policy for $\mathcal{I}_t$, $a_t^{i}$ is the $i$-th largest of all the arms, and the following inequality holds:
    \begin{align}
        X_{a_t^{i}} \hat{\beta}_{t-1}
        \geq
        X_{a_t^{*i}} \hat{\beta}_{t-1}
        \,,
    \end{align}
    for each $i \in [L]$.
    Then the regret in round $t$ is bounded as:
    \begin{align}
        \mathrm{reg}(t) 
        & = 
            \sum_{a_t^* \in \mathcal{I}_t^{*}}
            X_{a_t^*}^{t\top} \beta^* 
            -
            \sum_{a_t \in \mathcal{I}_t}
            X_{a_t}^{t\top} \beta^*
        \nonumber
        \\
        & = 
            \sum_{i=1}^{L}
            X_{a_t^{*i}}^{t\top} \beta^* 
            -
            X_{a_t^{i}}^{t\top} \beta^*
        \nonumber
        \\
        & = 
            \sum_{i=1}^{L}
              X_{a_t^{*i}}^{t\top} \beta^* 
            - X_{a_t^{*i}}^{t\top} \hat{\beta}_{t-1}
            + X_{a_t^{*i}}^{t\top} \hat{\beta}_{t-1}
            - X_{a_t^{i}}^{t\top} \beta^*
        \nonumber
        \\
        & \leq
            \sum_{i=1}^{L}
              X_{a_t^{*i}}^{t\top} \beta^* 
            - X_{a_t^{*i}}^{t\top} \hat{\beta}_{t-1}
            + X_{a_t^{i}}^{t\top} \hat{\beta}_{t-1}
            - X_{a_t^{i}}^{t\top} \beta^*
        \nonumber
        \\
        & =
            \sum_{i=1}^{L}
            X_{a_t^{*i}}^{t\top} (
                \beta^* - \hat{\beta}_{t-1}
            )
            + 
            X_{a_t^{i}}^{t\top} (
                \hat{\beta}_{t-1} - \beta^*
            )
        \nonumber
        \\
        & =
            \sum_{i=1}^{L}
            (X_{a_t^{*i}}^{t\top} - X_{a_t^i}^{t\top})
            (\beta^* - \hat{\beta}_{t-1})
        \nonumber
        \\
        & \leq 
            \sum_{i=1}^{L}
            \|X_{a_t^{*i}}^t - X_{a_t^i}^t\|_{\infty}
            \|\beta^* - \hat{\beta}_{t-1}\|_1
        \nonumber
        \\
        \label{eq:reward_beta_gap_comb}
        & \leq 
            2 L x_{\mathrm{max}}
            \|\beta^* - \hat{\beta}_{t-1}\|_1
        \,.
    \end{align}
    Below we also use $\mathrm{reg}(t) \leq 2 L x_{\mathrm{max}} b$, which can be derived from the first line of Eq.~\eqref{eq:reward_beta_gap_comb}.
    Then, the expected reward at round $t$ can be decomposed of 
    \begin{align}
        R(T) 
        & = 
            \sum_{t=1}^{T} \mathbb{E}[
                \mathrm{reg}(t)]
        \nonumber
        \\
        & = 
            \sum_{t=1}^{T_0} \mathbb{E}[
                \mathrm{reg}(t)]
            +
            \sum_{t=T_0+1}^{T} 
            \mathbb{E}[
                \mathrm{reg}(t)
                (
                I[ 
                    \mathcal{E}_{\mathrm{gap}}^{t-1\,c}
                ]
                +
                I[
                    \mathcal{E}_{\mathrm{gap}}^{t-1}
                ]
                )
            ]
        \nonumber
        \\
        & \leq 
            2 L x_{\mathrm{max}} b T_0
            +
            \sum_{t=T_0+1}^{T}
            2 L x_{\mathrm{max}} b
            P(
                \mathcal{E}_{\mathrm{gap}}^{t-1\,c}
            )
            +
            \mathbb{E}[
                \mathrm{reg}(t)
                I[ \mathcal{E}_{\mathrm{gap}}^{t-1}]
            ]
        \nonumber
        \\
        & \leq 
            2 L x_{\mathrm{max}} b T_0
            +
            2 L x_{\mathrm{max}} 
            \left(
            \sum_{t=T_0+1}^{T}
            b
            P(
                \mathcal{E}_{\mathrm{gap}}^{t-1\,c}
            )
            +
            \mathbb{E}[
                \|\beta^{*} - \hat{\beta}_{t-1}\|_1
                I[\mathcal{E}_{\mathrm{gap}}^{t-1}]
            ]
            \right)
        \nonumber
        \\
        & \leq 
            2 L x_{\mathrm{max}} b T_0
            +
            2 L x_{\mathrm{max}} 
            \left(
            \sum_{t=T_0}^{T}
            b
            P(
                \mathcal{E}_{\mathrm{gap}}^{t\,c}
            )
            +
            \mathbb{E}[
                \|\beta^{*} - \hat{\beta}_{t}\|_1
                I[\mathcal{E}_{\mathrm{gap}}^{t}]
            ]
            \right)
        \,.
    \end{align}
    In the fourth line, we use Eq.~\eqref{eq:reward_beta_gap_comb}.
    The rest of the proof is almost identical to the proof of Lemma \ref{lemma:greedy_bound_empirical}.
    Applying the inequality given by the event $\mathcal{E}_{\mathrm{gap}}^{t}$ to the last term, we obtain the following bound:
    \begin{align}
        R(T)
        & \leq 
            2 L x_{\mathrm{max}} b T_0
            +
            2 L x_{\mathrm{max}} 
            \left(
            \sum_{t=T_0}^{T}
            b
            P(
                \mathcal{E}_{\mathrm{gap}}^{t\,c}
            )
            +
            \frac{8 \lambda_{Lt} s_0}{\bar{\phi}^2}
            \mathbb{E}[
                I[\mathcal{E}_{\mathrm{gap}}^t]
            ]
            \right)
        \nonumber
        \\
        & \leq 
            2 L x_{\mathrm{max}} b T_0
            +
            2 L x_{\mathrm{max}} 
            \left(
            \sum_{t=T_0}^{T}
            b
            P(
                \mathcal{E}_{\mathrm{gap}}^{t\,c}
            )
            +
            \frac{8 \lambda_{Lt} s_0}{\bar{\phi}^2}
            \right)
        \nonumber
        \\
        & \leq
            2 L x_{\mathrm{max}} b T_0
            +
            2 L x_{\mathrm{max}} 
            \left(
            \sum_{t=T_0}^{T}
            b
            e^{-L t \kappa(\bar{\phi})^2}
            +
            2 
            b
            e^{-\delta^2/2}
            +
            \frac{
                32 s_0
                x_{\mathrm{max}} \sigma
            }{\bar{\phi}^2}
            \sqrt{\frac{\delta^2 + 2 \log d}{Lt}}
            \right)
        \,,
        \nonumber
    \end{align}
    where we substitute $\lambda_{Lt} = 4 x_{\mathrm{max}} \sigma \sqrt{(\delta^2 + 2 \log d)/ L t}$ in the last equality.
    Taking $\delta^2 = 4 \log (L t)$, we obtain:
    \begin{align}
        R(T) \leq
            2 L x_{\mathrm{max}} b T_0
            +
            2 L x_{\mathrm{max}} 
            \left(
            \sum_{t=T_0}^{T}
            b
            e^{-L t \kappa(\bar{\phi})^2}
            +
            \frac{2 b}{L^2 t^2}
            +
            \frac{
                32 s_0
                x_{\mathrm{max}} \sigma
            }{\bar{\phi}^2}
            \sqrt{\frac{4 \log (Lt) + 2 \log d}{Lt}}
            \right)
        \,.
        \nonumber
    \end{align}
    The second, third and fourth terms are upper bounded by:
    \begin{align}
        &
        \sum_{t=T_0}^T b e^{-L t \kappa(\bar{\phi})^2} 
            \leq
            \int_{0}^{\infty}
            b e^{-Lt\kappa(\bar{\phi})^2}
            dt
            =
            \frac{b}{L \kappa(\bar{\phi})^2}
        \,,
        \nonumber
        \\ 
        &
        \sum_{t=T_0}^T
        \frac{2 b}{L^2 t^2}
            \leq
            \sum_{t=1}^{\infty}
            \frac{2 b}{L^2 t^2}
            =
            \frac{\pi^2 b}{3 L^2}
        \,,
        \nonumber
        \\ 
        &
        \sum_{t=T_0}^T
        \sqrt{\frac{4\log (Lt) + 2\log d}{L t}}
        \leq
        \int_{0}^{T}
        \sqrt{\frac{4\log (LT) + 2\log d}{L t}}
        dt
        =
        2 \sqrt{
            \frac{(4\log (LT) + 2\log d)T}{L}
        }
        \,,
        \nonumber
    \end{align}
    respectively. Substituting $T_0 = \log (d (d+1)) / L \kappa(\bar{\phi})^2$, we finally obtain the following upper bound:
    \begin{align}
        R(T) 
        & \leq
            2 x_{\mathrm{max}} b
            \left(
            \frac{1 + \log (d (d+1))}{ \kappa(\bar{\phi})^2}
            +
            \frac{\pi^2}{3 L}
            \right)
            +
            \frac{
                128 s_0 x_{\mathrm{max}}^2 \sigma
            }{\bar{\phi}^2}
            \sqrt{
                \frac{ (4\log (LT) + 2\log d) T}{L}
            }
        \,.
    \end{align}
\end{proof}

\subsection{Proof of Theorem \ref{theorem:combinatorial}}
\begin{proof}
    \label{proof:combinatorial}
    We first redefine $\tilde{G}_{\beta}$ for the combinatorial setting as: 
    \begin{align}
        \tilde{G}_{\beta}
        :=
            \frac{1}{L}
            \sum_{
                |\mathcal{I}| = L
            }
            \int
            \left(
                \sum_{k \in \mathcal{I}}
                X_{k} X_{k}^{\top}
            \right)
            P(
                \text{Select $\mathcal{I}$} 
                \mid 
                \mathcal{X}, \beta
            ) 
            P(\mathcal{X}) 
            \prod_{k'}^{K} 
            dX_{k'}
            \,,
    \end{align}
    where $\mathcal{I} \subset [K]$ is a set of the arm indices, $\sum_{|\mathcal{I}| = L}$ is the sum over all possible combinations of $\mathcal{I}$ satisfying $|\mathcal{I}| = L$,
    and the arm selection policy is defined by:
    \begin{align}
        P(
            \text{Select $\mathcal{I}$} 
            \mid 
            \mathcal{X}, \beta
        ) 
        := 
            \frac{
                \prod_{i \in \mathcal{I}} 
                \prod_{j \in [K] \backslash \mathcal{I}} 
                I[
                    X_i^{\top} \beta
                    \geq 
                    X_j^{\top} \beta
                ]
            }{
                C
                (X_1^{\top} \beta, \dots, X_K^{\top} \beta)
            }
        \,.
        \label{eq:P_sel_comb}
    \end{align}
    Here
    $
        C(X_1^{\top} \beta, \dots, X_K^{\top} \beta)
    $
    is the normalization constant for the arm selection policy, which is a function of
    $X_1^{\top} \beta, \dots, X_K^{\top} \beta$.

    The expected Gram-matrix can be written as:
    \begin{align}
        \bar{G}_t 
        & := 
            \frac{1}{Lt} \sum_{s=1}^{t} 
            \mathbb{E} [
                \sum_{a_s \in \mathcal{I}_s}
                X_{a_{s}}^{s} X_{a_{s}}^{s\top} 
                \mid
                \mathcal{F}'_{s-1}
            ]
        \nonumber
        \\
        & = 
            \frac{1}{Lt} \sum_{s=1}^{t} 
            \sum_{|\mathcal{I}| = L}
            \int
                \left(
                    \sum_{k \in \mathcal{I}}
                    X_{k} X_{k}^{\top} 
                \right)
                P(
                    \text{Select $\mathcal{I}$} 
                    \mid 
                    \mathcal{X}, \hat{\beta}_{s-1}
                ) 
                P(\mathcal{X})
                \prod_{k'} dX_{k'}
        \nonumber
        \\
        & =
            \frac{1}{Lt} \sum_{s=1}^{t} 
            \tilde{G}_{\hat{\beta}_{s-1}}
        \,.
    \end{align}
    Therefore, similar to Lemma \ref{lemma:Gt}, it is sufficient to show that there exists $\phi_0 > 0$ that satisfies $\inf\limits_{\beta \in \mathbb{R}^d} \phi_{\mathcal{S}} (\tilde{G}_{\beta}) \geq \phi_0$.
    \begin{align}
        \tilde{G}_{\beta}
        & =
            \frac{1}{L}
            \sum_{
                |\mathcal{I}| = L
            }
            \int
            \left(
                \sum_{k \in \mathcal{I}}
                X_{k} X_{k}^{\top}
            \right)
            P(
                \text{Select $\mathcal{I}$} 
                \mid 
                \mathcal{X}, \beta
            ) 
            P(\mathcal{X}) 
            \prod_{k'}^{K} 
            dX_{k'}
        \nonumber
        \\
        & \succeq 
            \frac{1}{L}
            \sum_{
                |\mathcal{I}| = L
            }
            \int
            X_{i} X_{i}^{\top}
            I[i \in \mathcal{I}]
            P(
                \text{Select $\mathcal{I}$} 
                \mid 
                \mathcal{X}, \beta
            ) 
            P(\mathcal{X}) 
            \prod_{k'}^{K} 
            dX_{k'}
        \nonumber
        \\
        & =
            \frac{1}{L}
            \int
            X_{i} X_{i}^{\top}
            P(X_i) 
            P (\text{Select $i$} \mid X_i, \beta)
            dX_{i}
        \,.
    \end{align}
    In the second line, we only consider $X_i X_i^{\top}$ component and
    in the last line, we use
    \begin{align}
        P (\text{Select $i$} \mid X_i, \beta)
        & =
            \sum_{
                |\mathcal{I}| = L
            }
            \int
            I[i \in \mathcal{I}]
            P(
                \text{Select $\mathcal{I}$} 
                \mid 
                \mathcal{X}, \beta
            ) 
            P(\mathcal{X} \backslash \{X_i\})  
            \prod_{k' \neq i}^{K} 
            dX_{k'}
        \,.
    \end{align}
    We note that \textit{Select $i$} 
    in $P(\text{Select $i$} \mid X_i, \beta)$ 
    means that $i$ is in the top-$L$ set.
    From Eq.~\eqref{eq:P_sel_comb}, 
    $P(\text{Select $i$} \mid X_i, \beta)$ is a function of $(X_i)_{\beta}$ and therefore, we obtain
    \begin{align}
        \tilde{G}_{\beta}
        =
            \frac{1}{L}
            \int
            X_{i} X_{i}^{\top}
            f_{\beta} ((X_i)_{\beta})
            P_i (X_i)
            dX_{i}
    \end{align}
    where
    $
        f_{\beta} ((X_i)_{\beta})
        :=
        P (\text{Select $i$} \mid X_i, \beta)
    $.
    We note that by definition, 
    if $f_{\beta}(z) > 0$ for some $z$, 
    then $f_{\beta}(z') > 0$ for $z' > z$,
    and therefore 
    under Assumption \ref{assumption:selection_possibility},
    for any $\beta$, 
    $f_{\beta}$ is an element of $\mathcal{F}_1$ 
    defined in Eq.~\eqref{eq:function_set}.
    The above is the same form of 
    Eq.~\eqref{eq:tilde_G_lower_GM} in
    Lemma \ref{lemma:gm_lower_bound} 
    or 
    Eq.~\eqref{eq:tilde_G_lower_D} in
    Theorem \ref{theorem:radial}.
    Hence, 
    under Assumption \ref{assumption:selection_possibility},
    the same discussion in Section \ref{section:extension} holds when $P_i$ is the basis of the greedily-applicable distribution.
\end{proof}

\section{Specific cases for Theorem \ref{theorem:cc_mixture_general}}
Here we consider specific cases of Theorem \ref{theorem:cc_mixture_general}, i.e., situations where each arm follows an independent probability distribution, and each arm feature distribution can be written as a mixture of a distribution with good properties and the rest.
In such cases, the following remark holds.
\begin{remark}
    \label{theorem:cc_mixture_specific}
    Suppose that both an arm feature distribution $P(\mathcal{X})$ and 
    a $\phi_0$-greedy-applicable distribution $\tilde{P}(\mathcal{X})$ 
    are independent for each arm as 
    $P(\mathcal{X}) = \prod_{i=1}^{K} P_i (X_i)$, 
    and
    $\tilde{P}(\mathcal{X}) = \prod_{i=1}^{K} \tilde{P}_i (X_i)$, respectively.
    If 
    each $P_i(X_i)$ is a mixture of a PDF $Q_i(X_i)$ and $\tilde{P}_i(X_i)$, i.e., $P_i(X_i) = c_i \tilde{P_i} (X_i) + (1-c) Q_i(X_i)$ for a constant $0<c_i<1$ for each $i \in [K]$, then $P(\mathcal{X})$ is a $(\prod_{i=1}^{K}c_i)\phi_0$-greedy-applicable distribution.
\end{remark}
\begin{proof}
    From the definition, we have
    \begin{align}
        \bar{G}_t (P)
        & =
            \frac{1}{t}
            \sum_{s=1}^{t}
            \mathbb{E}
            \left[
                X_{a_t}^{s}
                X_{a_t}^{s\top}
                \mid
                \mathcal{F}'_{s-1}
            \right]
        \nonumber
        \\
        & =
            \frac{1}{t}
            \sum_{s=1}^{t}
            \sum_{k=1}^{K}
            \int
            X_k X_k^{\top} 
            P(\text{Select $k$} \mid \mathcal{X}, \hat{\beta}_{s-1})
            P (\mathcal{X})
            \prod_{k'=1}^{K} dX_{k'}
        \nonumber
        \\
        & =
            \frac{1}{t}
            \sum_{s=1}^{t}
            \sum_{k=1}^{K}
            \int
            X_k X_k^{\top} 
            P(\text{Select $k$} \mid \mathcal{X}, \hat{\beta}_{s-1})
            \prod_{k'=1}^{K} 
            P_{k'} (X_{k'})
            dX_{k'}
        \nonumber
        \\
        & =
            \frac{1}{t}
            \sum_{s=1}^{t}
            \sum_{k=1}^{K}
            \int
            X_k X_k^{\top} 
            P(\text{Select $k$} \mid \mathcal{X}, \hat{\beta}_{s-1})
            \prod_{k'=1}^{K} 
            (
                c_{k'} \tilde{P}_{k'} (X_{k'})
                +
                (1-c_{k'}) Q_{k'} (X_{k'})
            )
            dX_{k'}
        \nonumber
        \\
        & \succeq
            \left(\prod_{k'=1}^{K} c_{k'}\right)
            \frac{1}{t}
            \sum_{s=1}^{t}
            \sum_{k=1}^{K}
            \int
            X_k X_k^{\top} 
            P(\text{Select $k$} \mid \mathcal{X}, \hat{\beta}_{s-1})
            \prod_{k'=1}^{K} 
            \tilde{P}_{k'} (X_{k'})
            dX_{k'}
        \nonumber
        \\
        & =
            \left(\prod_{k'=1}^{K} c_{k'}\right)
            \bar{G}_t (\tilde{P})
        \,.
    \end{align}
    By the definition of $\tilde{P}(\mathcal{X})$, we see $\phi_{\mathcal{S}}(\bar{G}_t(P)) \geq (\prod_{i=1}^{K} c_i) \phi_{\mathcal{S}}(\bar{G}_t(\tilde{P})) = (\prod_{i=1}^{K} c_i) \phi_0$ from the last line.
\end{proof}
Particularly, if all arm feature distributions are the same distribution, then the following obviously holds:
\begin{remark}
    Assume that
    both an arm feature distribution $P(\mathcal{X})$ and 
    a $\phi_0$-greedy-applicable distribution $\tilde{P}(\mathcal{X})$ 
    are independent for each arm and the distribution of each arm is identical:
    $P(\mathcal{X}) = \prod_{i=1}^{K} P' (X_i)$, 
    and
    $\tilde{P}(\mathcal{X}) = \prod_{i=1}^{K} \tilde{P}' (X_i)$, respectively.
    If $P'(X)$ is a mixture of a PDF $Q(X)$ and $\tilde{P}'(X)$, i.e., $P'(X) = c \tilde{P} (X) + (1-c) Q(X)$ for a constant $0<c<1$, then $P(\mathcal{X})$ is a $c^K \phi_0$-greedy-applicable distribution.
\end{remark}

\section{Approximations}
The claim of Theorem \ref{theorem:cc_mixture_general} is for the case where arm feature distribution $P(X)$ has a mixture component of a greedy-applicable distribution $\tilde{P}(X)$. 
Intuitively, one can infer that a PDF that is well approximated by a greedy-applicable distribution (e.g., with an accuracy of $O(\epsilon)$) is still often a greedy-applicable distribution.
In fact, the conjecture turns out to be true if both $X \sim {P}(X)$ and $X \sim \tilde{P}(X)$ are bounded:
\begin{theorem}
    \label{theorem:approx}
    Let us write the error of the approximator $\tilde{P}(X)$ for the original probability distribution $P(X)$ as $||P-\tilde{P}||_{L^{1}} < \epsilon $, where the $L^1$-norm for a function $f: \mathbb{R}^d \rightarrow \mathbb{R}$ is defined by $||f||_{L^1} := \int_{\mathbb{R}^d} |f(X)| dX$. 
    Under Assumptions
    \ref{assumption:x_beta_bound} and
    \ref{assumption:arm_selection}, 
    and
    if $X \sim \tilde{P} (X)$ is bounded as $\|X\|_{\infty} < x_{\mathrm{max}}$,
    then,
    $\phi_{\mathcal{S}}(\bar{G}_t(P)) \geq \phi_{\mathcal{S}}(\bar{G}_t(\tilde{P})) - O(\epsilon)$.
\end{theorem}
\begin{proof}
    \label{proof:approx}
    From the definition, we have
    \begin{align}
        \bar{G}_t (P)
        & =
            \frac{1}{t}
            \sum_{s=1}^{t}
            \mathbb{E}
            \left[
                X_{a_t}^{s}
                X_{a_t}^{s\top}
                \mid
                \mathcal{F}'_{s-1}
            \right]
        \nonumber
        \\
        & =
            \frac{1}{t}
            \sum_{s=1}^{t}
            \sum_{k=1}^{K}
            \int
            X_k X_k^{\top} 
            P(\text{Select $k$} \mid \mathcal{X}, \hat{\beta}_{s-1})
            P (\mathcal{X})
            \prod_{k'=1}^{K} dX_{k'}
        \nonumber
        \\
        & =
            \frac{1}{t}
            \sum_{s=1}^{t}
            \sum_{k=1}^{K}
            \int
            X_k X_k^{\top} 
            P(\text{Select $k$} \mid \mathcal{X}, \hat{\beta}_{s-1})
            (
                \tilde{P}(\mathcal{X})
                +
                (
                    P(\mathcal{X})
                    -
                    \tilde{P}(\mathcal{X})
                )
            )
            \prod_{k'=1}^{K} dX_{k'}
        \nonumber
        \\
        & \succeq 
            \frac{1}{t}
            \sum_{s=1}^{t}
            \sum_{k=1}^{K}
            \int
            X_k X_k^{\top} 
            P(\text{Select $k$} \mid \mathcal{X}, \hat{\beta}_{s-1})
            \tilde{P}(\mathcal{X})
            \prod_{k'=1}^{K} dX_{k'}
        \nonumber
        \\
        & \hspace{50pt}
            -
            \frac{1}{t}
            \sum_{s=1}^{t}
            \sum_{k=1}^{K}
            \int
            X_k X_k^{\top} 
            P(\text{Select $k$} \mid \mathcal{X}, \hat{\beta}_{s-1})
            |
                P(\mathcal{X})
                -
                \tilde{P}(\mathcal{X})
            |
            \prod_{k'=1}^{K} dX_{k'}
        \nonumber
        \\
        & = 
            \bar{G}_t (\tilde{P})
            -
            \frac{1}{t}
            \sum_{s=1}^{t}
            \sum_{k=1}^{K}
            \int
            X_k X_k^{\top} 
            P(\text{Select $k$} \mid \mathcal{X}, \hat{\beta}_{s-1})
            |
                P(\mathcal{X})
                -
                \tilde{P}(\mathcal{X})
            |
            \prod_{k'=1}^{K} dX_{k'}
        \nonumber
        \\
        & \succeq 
            \bar{G}_t (\tilde{P})
            -
            \frac{1}{t}
            \sum_{s=1}^{t}
            \sum_{k=1}^{K}
            \int
            X_k X_k^{\top} 
            |
                P(\mathcal{X})
                -
                \tilde{P}(\mathcal{X})
            |
            \prod_{k'=1}^{K} dX_{k'}
            \,.
        \label{eq:Gt_approx}
    \end{align}
    Because $P$ and $\tilde{P}$ are bounded, we obtain
    \begin{align}
        \bar{G}_t (P)
        & \succeq 
            \bar{G}_t (\tilde{P})
            -
            K x_{\max}^2 1_{d \times d}
            \int
            |
                P(\mathcal{X})
                -
                \tilde{P}(\mathcal{X})
            |
            \prod_{k'=1}^{K} dX_{k'}
        \nonumber
        \\
        & =
            \bar{G}_t (\tilde{P})
            -
            \epsilon K x_{\max}^2 1_{d \times d}
        \,.
    \end{align}
    Here $1_{d\times d} \in \mathbb{R}^{d \times d}$ represents the all ones matrix.
    For any $V \in \mathbb{R}^d$, we have
    \begin{align}
        V^{\top} \bar{G}_t (P) V
        \geq 
        V^{\top} \bar{G}_t (\tilde{P}) V
        -
        \epsilon K x_{\max}^2 (\sum_d (V)_d)^2
        \geq 
        V^{\top} \bar{G}_t (\tilde{P}) V
        -
        \epsilon K x_{\max}^2 \|V\|_1^2
        \,,
    \end{align}
    which concludes
    $
        \phi_{\mathcal{S}}(\bar{G}_t(P)) 
        \geq 
        \phi_{\mathcal{S}}(
            \bar{G}_t(\tilde{P})
        ) - O(\epsilon)
    $.
\end{proof}

\section{An application to the dense parameter setting}
In the non-sparse case, the analysis of the greedy algorithm has also been performed by \citet{bastani2021mostly}.
In this work, a more abstract assumption is employed:
\begin{assumption}{Covariate Density}
    \label{assumption:CD}
    For a vector $\beta \in \mathbb{R}^{d}$,
    there exists $\lambda_0 > 0$ such that 
    \begin{align}
        \lambda_{\mathrm{min}}
        \left(
            \mathbb{E} [XX^{\top} I[X^{\top} \beta > 0]]
        \right)
        \geq \lambda_0
        \,,
    \end{align}
    where $\lambda_{\mathrm{min}}(A)$ represents the minimum eigenvalue of $A \in \mathbb{R}^{d \times d}$. 
\end{assumption}
In this study, Assumption \ref{assumption:RS} is mentioned as one of the sufficient conditions.
As with the sparse setting, our analysis extends the applicability of this theory. For this setting, we make the following assumption, which corresponds to Assumption \ref{assumption:selection_possibility}:
\begin{assumption}
    \label{assumption:selection_possibility_dense}
    We assume
    $
    \inf_{\beta \in \mathbb{R}^d} 
    \int I[X^{\top} \beta \geq 0] P(X) dX > 0
    $.
\end{assumption}
Then, we find
\begin{theorem}
    \label{theorem:CD}
    Under Assumption 
    \ref{assumption:selection_possibility_dense}, 
    if 
    $P(X)$ is described by $P_{GM}$, $P_{LGM}$, $P_{D}$, or $P_{R}$ with the operator $\phi_{\mathcal{S}}$ replaced by $\lambda_{\mathrm{min}}$,
    then there exists $\lambda_0 > 0$ that satisfies Assumption \ref{assumption:CD}.
\end{theorem}
\begin{proof}
    \label{proof:CD}
    In the proof of all theorems in Section \ref{section:extension}, we proved the lower bound for $\phi_{\mathcal{S}}$ by using the inequality $\succeq$ on the positive definiteness of matrices.
    In other words, the same theorem can be proved by replacing the operator $\phi_{\mathcal{S}}$ by $\lambda_{\min}$.

    The expectation in Assumption \ref{assumption:CD} is identical to $\tilde{G}_{\beta}$ in the greedy policy with $K=2$ and the arm features generated independently by $P(X_1)$ and $\delta(X_2)$, respectively:
    \begin{align}
        \tilde{G}_{\beta}
        & =
            \int 
            (
                X_1 X_1^{\top}
                I[X_1^{\top} \beta \geq X_2^{\top} \beta]
                +
                X_2 X_2^{\top}
                I[X_2^{\top} \beta \geq X_1^{\top} \beta]
            )
            P(X_1) P(X_2)
            dX_1 dX_2
        \nonumber
        \\
        & =
            \int 
            X_1 X_1^{\top}
            I[X_1^{\top} \beta \geq 0]
            P(X_1)
            dX_1
        \,.
    \end{align}
    Assumption \ref{assumption:selection_possibility} can be written as
    \begin{align}
    \inf_{\beta \in \mathbb{R}^d} 
    P(\text{Select $i$} \mid \beta) 
    = 
    \inf_{\beta \in \mathbb{R}^d} 
    \int I[X^{\top} \beta \geq 0] P(X) dX 
    > 0
    \,,
    \end{align}
    which is the same statement as Assumption \ref{assumption:selection_possibility_dense}.
    Therefore, under the conditions of the theorem, the same discussion holds as in Section \ref{section:extension} with the operator $\phi_{\mathcal{S}}$  replaced by $\lambda_{\min}$.
\end{proof}


\end{document}